%% file: main.tex
\documentclass{article}[12pt]
\usepackage[utf8]{inputenc}
\usepackage{amsmath, bbm}
\usepackage{fullpage}
\usepackage{hyperref}
\hypersetup{colorlinks,
            linkcolor=red,
            citecolor=blue,
            urlcolor=magenta,
            linktocpage,
            plainpages=false}
\usepackage{algorithm}
\usepackage{algpseudocode}
\usepackage{amsfonts,amssymb}
\usepackage{mathtools}
\usepackage{amsthm} % better to load after ams math
\usepackage[x11names]{xcolor}
\usepackage{header}
\usepackage{natbib, xcolor}

\title{A Nonstochastic Control Approach to Optimization \\
{\small preliminary version} }
\author{Xinyi Chen \thanks{Princeton University and Google DeepMind} \and Elad Hazan \thanks{Princeton University and Google DeepMind}}

\begin{document}

\maketitle

\begin{abstract}

Selecting the best hyperparameters for a particular optimization instance, such as the learning rate and momentum, is an important but nonconvex problem. As a result, iterative optimization methods such as hypergradient descent lack global optimality guarantees in general. 

We propose an online nonstochastic control methodology for mathematical optimization. First, we formalize the setting of meta-optimization, an online learning formulation of  learning the best optimization algorithm from a class of methods. The meta-optimization problem over gradient-based methods can be framed as a feedback control problem over the choice of hyperparameters, including the learning rate, momentum, and the preconditioner.

Although the original optimal control problem is nonconvex, we show how recent methods from online nonstochastic control using convex relaxations can be used to overcome the challenge of nonconvexity, and obtain regret guarantees against the best offline solution. This guarantees that in meta-optimization, given a sequence of optimization problems, we can learn a method that attains convergence comparable to that of the best optimization method in hindsight from a class of methods\footnote{A conference version of this manuscript, titled "Online Control for Meta-optimization", is accepted to Neural Information Processing Systems (NeurIPS) 2023.}.

\end{abstract}

\newpage
\tableofcontents
\newpage

\input{intro}
\input{opt_and_control.tex}

\input{dynamical_system.tex}

\input{alg-n-thm.tex}
\input{smooth_convex.tex}
\input{proofs}

\input{conclusion.tex}

\bibliographystyle{plainnat}
\bibliography{refs}
\newpage
\appendix

\input{bandit_ocowm}

\end{document}

%% file: intro.tex
\section{Introduction.}

The performance of optimization methods, in particular for the application of training deep neural networks, crucially depends on their hyperparameters. \citet{bengio2012practical} notes that the learning rate is the ``single most important parameter" for efficient training. However, it is also notoriously hard to tune without intensive hyperparameter search. 

We formalize the problem of hyperparameter optimization as the task of {\bf meta-optimization}. In this task, the player is given a sequence of optimization problems. Her goal is to solve them as fast as the best algorithm in hindsight from a family of possible methods. This setting generalizes hyperparameter tuning to encompass several attributes, including the learning rate, preconditioner, momemtum, and more. 

The problem of meta-optimization, even for the special case of hyperparameter tuning of the learning rate, is difficult because {\bf optimizing these hyperparameters can be nonconvex}. In this paper, we investigate how to overcome this nonconvexity using recent tools from control theory, namely the online nonstochastic control framework. Our main result is an efficient method for meta-optimization which has provable guarantees: over a sequence of optimization problems, the method approaches the performance of the best algorithm in hindsight from a class of methods.

\ignore{
\subsection{Mathematical optimization and feedback control}

The fields of mathematical optimization and control theory are closely related. Many optimization methods are inspired by natural dynamical systems, such as Polyak's heavy ball method \citep{polyak1964some}.  On the other hand, the analysis of optimization algorithms also has fundamental connections with the mathematics of dynamical systems. Besides the analysis of the heavy ball method, \citep{wang2021modular}, other examples where optimization algorithms are analyzed as dynamical systems include Nesterov momentum for smooth functions \citep{su2014differential,muehlebach2019dynamical}, and more recently a frameworks of analysis using Lyapunov's theory of stability \citep{iqc, wilson2018lyapunov}.

% By far the most comprehensive use of control in optimization is that of contractivity theory. Lyapunov's direct method (see e.g. \cite{slotine1991applied}) is used to design an energy function, which is shown to contract with the evolution of a given optimization algorithm. This technique has been used to analyze meny existing optimization algorithms. 

By far the most widely used framework of control in optimization is that of Lyapunov’s stability method, also known as Lyapunov’s direct method. This method (see e.g. \cite{slotine1991applied}) proposes that if the system’s energy dissipates over time, then the system must be stable in some sense.  % The use of an energy function makes the analysis of dynamical systems more tractable, since computing the trajectories of high-dimensional or highly nonlinear systems can be cumbersome. 
For a given optimization algorithm, if there exists a quantity that contracts with the evolution of the algorithm, Lyapunov’s direct method can be applied to show convergence. This technique has been used to analyze many existing optimization algorithms.

% In this paper we consider using Lyapunov's second technique, the so called ``direct method". On one hand, this technique is more limited: it is unique to time-varying linear dynamical systems. On the other hand, it encompasses not only stability, but also optimal control. This additional power can potentially guarantee convergence to the optimal method in optimization, which is exactly our goal.

% However, Lyapunov's direct method is applicable to optimal control, and optimal feedback control in the context of optimization is a nonconvex problem. We circumvent this diffulty using new techniques in control theory devised in the context of machine learning: online nonstochastic control. The latter framework circumvents computational hardness by means of convex relaxation and improper learning, to give new regret guarantess for online control. 

% We show how to reduce meta optimization to online nonstochastic control, and derive efficient algorithms for this optimization framework. 

However, Lyapunov’s direct method does not provide optimality guarantees for control: prescriptive suggestions for how to drive the system to the target state. Moreover, standard Lyapunov analysis does not take disturbances to the system into account, even though they naturally arise in most physical systems \footnote{An exception is the work of \cite{iqc} which does tollerate some amonut of disturbance with provable guarantees.}.  To overcome these limitations, we consider Lyapunov's second technique, the so-called ``indirect method". Instead of working directly with the nonlinear system, this method studies the behavior of the system around the linearization about the equilibrium point. This method is more amenable to optimal control theory, which can potentially guarantee convergence to the best optimization algorithm.

% On one hand, this technique is more limited: the analysis only applies to time-varying linear dynamical systems. On the other hand, while it is an approach to determine stability, it is also amenable to optimal control theory. This additional power can potentially guarantee convergence to the optimal method in optimization, which is exactly our goal. 

Still, optimal control has several limitations in this context of learning optimizers. Methods for optimal control, notably LQR theory, assume access to the dynamics a priori, stochastic noise, and are limited to quadratic cost functions. Further, learning the best feedback controller directly is a nonconvex problem.  
%Even though optimal control can be applied to linearizations of nonlinear systems, this problem in the context of optimization is nonconvex. \xc{here we talk about linearizations of objective functions, but later the example is for the original objective functions} 
We circumvent these difficulties by using new techniques in control theory devised in the context of machine learning. The online nonstochastic control framework bypasses the computational hardness issue prevalent in nonconvex problems by using convex relaxation and improper learning, leading to new regret guarantees for online control. It is applicable for an online setting, where the costs and dynamics are unknown a priori, and for general convex costs. Moreover, this framework provides guarantees even in the presence of adversarial disturbances -- disturbances that do not follow distributional assumptions.  As far as we know, optimal control theory has not been systematically applied to optimization largely due to the limitations above \footnote{see related work section.}. 
}

\subsection{The setting of meta-optimization.}

In meta-optimization, we are given \textbf{a sequence of optimization problems}, called episodes. The goal of the player is not only to minimize the total cost in each episode, but also to compete with a set of available optimization methods. %for example competing with the best global choice of hyperparameters. 

Specifically, in each of the $N$ episodes, we have a sequence of $T$ optimization steps that are either deterministic, stochastic, or online. Throughout the paper, we use $(t, i)$ to denote time $t$ in episode $i$. At the beginning of an episode, the iterate is ``reset" to a given starting point $x_{1, i}$. In the most general formulation of the problem, at time $(t, i)$, an optimization algorithm $\mA$ chooses a point $x_{t, i} \in \K$, in a convex domain $\K \subseteq \reals^d$. It then suffers a cost $f_{t, i}(x_{t, i})$. Let $x_{t, i}(\mA)$ denote the point chosen by $\mA$ at time $(t, i)$, and for a certain episode $i \in [N]$, denote the cost of an optimization algorithm $\mA$ by
$$ J_i(\mA) = \sum_{t=1}^T f_{t, i}(x_{t, i}(\mA) ). $$ % $$- \min_{x^* \in \K} \sum_t f_t^i(x^*) . $$
The protocol of this setting is formally defined in Algorithm \ref{alg:episodic_setting}.
\begin{algorithm}[H]
\caption{Meta-optimization }
\label{alg:episodic_setting}
\begin{algorithmic}[1]
\State Input: $N, T, \mA, \K$, reset points $\{x_{1, i}\}_{i=1}^N$.
\For{$i = 1, \ldots, N$}
\For{$t = 1, \ldots, T$}
\State Play $x_{t, i} = \mA (f_{1, 1}, \ldots, f_{t-1, i}) \in \K$ if $t > 1$; else play $x_{1, i}$.
\State Receive $f_{t, i}$, pay  $ f_{t, i}(x_{t, i})$
\EndFor
\EndFor
\end{algorithmic}
\end{algorithm}

The standard goal in optimization, either deterministic, stochastic, or online, is to minimize each $J_i$ in isolation.  In meta-optimization, the goal is to minimize the cumulative cost in both each episode, and overall in terms of the choice of the algorithm. We thus define the meta-regret to be
$$\mregret(\A) = \sum_{i=1}^N  J_i(\mA) - \min_{\A^* \in \Pi} \sum_{i=1}^N J_i(\A^*) ,$$ where $\Pi$ is the benchmark algorithm class.
The meta-regret is the regret incurred for learning the best optimization algorithm in hindsight, and captures both efficient per-episode optimization as well as competing with the best algorithm in $\Pi$.

\subsubsection*{Why is meta-optimization hard?} It is natural to apply standard techniques, such as local search or gradient based methods, to meta-optimization. Indeed, this has been studied in numerous previous works under related settings, e.g. in \citet{baydin2017online,chandra2019gradient}. 

However, the resulting optimization problem is nonconvex and local gradient optimization may reach suboptimal solutions. 
For example, a natural application of meta-optimization is to learn hyperparamters of a class of optimization algorithms. In particular, finding the optimal gradient descent learning rate for a specific objective takes the form of the following minimization problem:
$$\min_\eta f(x_T),\ \ \text{subject to } x_{t+1} = x_t - \eta\nabla f(x_t).
$$
We can unroll the minimization objective from the initial point $x_1$,
\begin{align*}
f(x_T) &= f(x_1 - \eta\nabla f(x_1) -  \eta\sum_{t=2}^{T-1} \nabla f(x_t)) = f(x_1 - \eta\nabla f(x_1) - \eta \nabla f(x_1 - \eta\nabla f(x_1)) - ...)
\end{align*}
From this expression, it is clear that $f(x_T)$ is a nonconvex function of $\eta$. The task becomes even harder when we consider natural variants of the problem, including changing loss functions, stochastic gradients, optimizing over preconditioners instead of scalar learning rates, adding momentum, Nesterov acceleration, and so forth. \nocite{polyak1964some,nesterov1983method}

The setting of meta-optimization generalizes certain prior approaches for hyperparamter tuning. For example, it is the online learning analogue of the average case analysis framework of optimization algorithms \citep{average_case,average_case_sgd}, and generalizes control-based approaches for analyzing a single optimization instance \citep{iqc,pmlr-v134-casgrain21a}.

% , and hypergradient descent approaches \citep{baydin2017online,wang2021guarantees}.

\subsection{Our contributions.}
\label{subsec:contribution}

Applying iterative optimization methods directly to the meta-optimization problem is unlikely to result in global optimality, since it is nonconvex. 
We take a different approach: 
we formulate meta-optimization as an online nonstochastic control problem. Using recent techniques from this framework, we give an efficient algorithm for unconstrained meta-optimization of quadratic and convex smooth functions. %This algorithm iteratively updates its parameters, and uses them to implement a changing optimization method for the underlying episodes.

The following is an informal statement of our main theorems. %, which can be applied to different optimization settings such as stochastic, online or deterministic, to derive provable guarantees. 

\vspace{1.5mm}
\fbox{\parbox{0.95\textwidth}{%
\begin{theorem}[Informal] \label{thm:main_intro} There is an efficient algorithm $\A$\footnotemark\ for convex quadratic meta-optimization that satisfies \\
$$ \mregret(\A) = \sum_{i=1}^N\sum_{t=1}^T f_{t, i}(x_{t, i}^{\A}) -\min_{\A^*\in \Pi}\sum_{i=1}^N\sum_{t=1}^T f_{t, i}(x_{t, i}^{\A^*}) \le \tilde{O}(\sqrt{N T}). $$
\end{theorem}

\begin{theorem} [Informal] For convex smooth losses, a bandit variant of the above algorithm satisfies \\ 
\[ \E\left[\mregret(\A)\right] \le \tilde{O}((N T)^{3/4}).\]
\end{theorem}
}}
\vspace{1.5mm}
\footnotetext{Algorithm \ref{alg:episodic} in Section \ref{sec:thm+alg}}

In the statements above, $\tilde{O}$ hides factors polynomial in parameters of the 
problem and logarithmic in $T, N$ and $\Pi$ is the benchmark class of algorithms. This general guarantee implies convergence to optimality in deterministic and stochastic optimization, as well as regret minimization in the online learning setting. 
As the number of episodes increases, the average performance approaches that of the best optimizer from a family of optimization methods in hindsight. 

\subsubsection{The control formulation.}
We describe in more detail our formulation of optimization as a control problem in the simplest setting -- the deterministic setting, where the objective function is invariant over time steps and episodes. Let the objective function be $f: \reals^d \rightarrow \reals$. The state of the most general dynamical system we can formulate at time $(t, i)$ is given by the vector below which contains $h$ past iterates and gradients:
$$
z_{t, i} = \begin{bmatrix}
x_{t+1, i}\\
\vdots \\
x_{t-h+2, i} \\ 
\nabla f_{t, i}(x_{t, i}) \\
\vdots \\
\nabla f_{t-h+1, i}(x_{t-h+1, i})
\end{bmatrix}.  \mbox{  For simplicity, we consider  } 
 z_{t, i} = \begin{bmatrix}
x_{t, i} \\
x_{t-1, i}\\
\nabla f(x_{t-1, i})
\end{bmatrix},$$
consisting of the current optimization iterate, the previous iterate, and the previous gradient. Let $H$ denote the Hessian of $f$, $\delta > 0$ be a regularization parameter, and $\eta >0$ be the base learning rate, then the system we consider evolves according to the following dynamics:
\begin{align*}
\begin{bmatrix}
x_{t+1, i} \\
x_{t, i}\\
\nabla f(x_{t, i})
\end{bmatrix} 
= 
\begin{bmatrix}
(1-\delta)I & 0   &  -\eta I\\
I & 0 & 0 \\
 H & -H & 0
\end{bmatrix} 
\times
\begin{bmatrix}
x_{t, i} \\
x_{t-1, i}\\
\nabla f(x_{t-1, i})
\end{bmatrix} 
+ 
\begin{bmatrix}
I & 0 & 0 \\
0 & 0 & 0\\
0 & 0 & 0
\end{bmatrix}  
\times
u_{t, i} 
 + \begin{bmatrix}
 0\\
 0\\
 \nabla f(x_{t-1, i})
 \end{bmatrix}.
 \end{align*}
 We can write the equation above as
$$
z_{t+1, i}  =  Az_{t, i} + Bu_{t, i} + w_{t, i},$$
where $A$ is the dynamics matrix, $B$ is the control-input matrix, and the $w$'s are the disturbances. This is a non-standard use of disturbance in control, where it is usually unknown to the controller. However, this formulation was chosen to ensure stability and leads to our theoretical guarantees \footnote{Notice that the regret bounds we prove are under the same sequence of disturbances, and apply to the gradients along the taken trajectory. This is similar to the nature of adaptive gradient methods: the best regularization in hindsight depends on the observed gradients, rather than the ones that would appear had we used a different algorithm.}. 

Without the control signal $u_{t, i}$, the system describes the partially time-delayed gradient descent update with $\ell_2$ regularization: 
$
x_{t+1, i} = (1-\delta)x_{t, i} - \eta \nabla f(x_{t-1, i}).$
Several results show that under mild conditions, with a delay of one iteration, the regret using delayed gradients are constant factors away from the same algorithm using using fresh gradients \citep{delay, delay1}. We demonstrate in later sections that $\delta$ is a user-specified parameter, and can be arbitrarily small. 

The control signal $u_{t, i}$ contributes only to the gradient update, and can be used to simulate the update of any optimization method. Finally, observe that for a quadratic function $f$, the evolution of the gradient follows
$$
\nabla f(x_{t, i}) = \nabla f(x_{t-1, i}) + H(x_{t, i} - x_{t-1, i}),$$
and we incorporate $\nabla f(x_{t-1, i})$ in the adversarial disturbance. Importantly, this formulation can capture the state reset when a new episode begins by using a specific disturbance, as we elaborate in Section \ref{sec:dynamics}.

\subsubsection{The benchmark algorithm class.}

Informally, our guarantee is competitive with optimizers that are linear functions of past gradients. We focus on the deterministic setting for the rest of the section, since in this case the benchmark algorithm class has a more straightforward interpretation, and provide the full description in Section \ref{sec:benchmark}. The class consists of optimization algorithms parameterized by a matrix $K = [K_1\ K_2\ K_3]\in \reals^{d\times 3d}$. Let $x_{t, i}$ be the iterates played by our algorithm, and
$$
z_{t, i}^K = \begin{bmatrix}x_{t, i}^K\\
x_{t-1, i}^K\\
\hat{\nabla}f(x_{t-1, i}^K)
\end{bmatrix},\ \ \ \text{where  }\hat{\nabla}f(x_{t-1, i}^K) = \nabla f(x_{t-1, i}^K) - \nabla f(x_{t-2, i}^K) + \nabla f(x_{t-2, i}),
$$
be the state at time $(t, i)$ reached by the optimizer parameterized by $K$. The optimizer with parameter $K$ has the corresponding updates:
\begin{align} \label{eq:lin_policies} x_{t+1, i}^K = ((1-\delta)I + K_1)x_{t, i}^K + K_2 x_{t-1, i}^K + (K_3 - \eta I)\hat{\nabla}f(x_{t-1, i}^K).\end{align} 
% Then in the deterministic setting, our algorithm can compete with the class of optimizers $\Pi$ of the form
% \begin{align*}
% x_{t+1, i}^K &= (1-\delta)x_{t, i}^K - \eta \hat{\nabla} f(x_{t-1, i}) + K z_{t, i}^K
% \end{align*}
% where the linear transformation 

 This class $\Pi$ can capture common optimization algorithms on time-delayed pseudo-gradients $\hat{\nabla}f$ for the deterministic setting, and we give some examples in the table above with the corresponding choice of $K$. The class also includes any combination of the methods in the table, which can be expressed by an appropriate choice of $K$.  For more details on the class of algorithms and restrictions on $K$, see Section \ref{sec:policy_class}; for a concrete example of learning the learning rate, see Section \ref{sec:example}.
\begin{table}
\begin{tabular}{ |c| l| l|}
 \hline
 Method & K & Update\\ 
  \hline GD with learning rate $\eta'$ & $\begin{bmatrix} 0 & 0 & (\eta - \eta') I\end{bmatrix}$ & $x_{t+1, i}^K = (1-\delta)x_{t, i}^K - \eta' \hat{\nabla} f(x_{t-1, i})$
     \\ \hline
 Momentum  &  $\begin{bmatrix} -\beta I & \beta I & 0 \end{bmatrix}$ & $x_{t+1, i}^K = (1-\delta-\beta)x_{t, i}^K +\beta x_{t-1, i}^K -\eta\hat{\nabla} f(x_{t-1, i})$ \\   \hline
 Preconditioned methods &$\begin{bmatrix} 0 & 0 & \eta I - P \end{bmatrix}$  & $x_{t+1, i}^K = (1-\delta)x_{t, i}^K -P\hat{\nabla} f(x_{t-1, i})$\\
 \hline
\end{tabular}
\end{table}

% In settings where the objective functions are time-varying, the benchmark algorithm class is instead specified by the disturbances realized along the actual trajectory of optimization. 

% the optimization algorithm $K = \{K^l\}_{l=1}^L$ of the form
% $$x_{t+1, i}^K = (1-\delta)x_{t, i}^K - \eta \hat{\nabla} f(x_{t-1, i}) + \sum_{l=1}^L K^l\nabla f(x_{t-l, i}).
% $$
% Then our algorithm can compete with all such optimizers parameterized by $K$, where $K$ satisfies certain norm constraints. Note that the past gradients $\nabla f(x_{t-l, i})$ in the update above do not have the superscript $K$, and they are gradients taken along the trajectory of our algorithm's iterates, instead of the counterfactual trajectory under the optimizer $K$. 
 
%  This class of optimizers can approximate, in terms of total cost, the class of 
% $$ x_{t+1, i}^K = (1-\delta)x_{t, i}^K - \eta \nabla f(x_{t-1}) + K_1x_t + K_2 x_{t-1} + K_3 \nabla f(x_{t-1}). $$
% for a linear tranformation $K$. 
% \eh{put some table with methods, momentum,... and spell out K for them}

\subsubsection{Guarantees for different optimization settings.}\label{sec:optimization_settings}
\paragraph{Deterministic optimization.} In this setting, we are given a fixed objective function, i.e. $f_{t, i} = f$ for all $t,i$. Let $\bar{x} = \frac{1}{TN}\sum_{i=1}^N\sum_{t=1}^T x_{t, i} $ be the average iterate, and $\bar{J}(\A) = \frac{1}{TN}\sum_{i=1}^N J_i(\A)$ denote the average cost of the optimization algorithm $\A$. We also drop the superscript $\A$ on the left hand side for clarity of notation. Theorem \ref{thm:main_intro} guarantees
$$ f(\bar{x}) \leq \min_{\A^* \in \Pi} \bar{J}(\A^*) + \tilde{O}\left(\frac{1}{\sqrt{TN}}\right). $$
That is, the function value of the average iterate over all episodes approaches the average cost of the best optimization algorithm from $\Pi$. Here $\Pi$ refers to the class of algorithms described by (\ref{eq:lin_policies}). 

\paragraph{Stochastic optimization.}
% Suppose our functions are drawn randomly from distributions $\D_1, \D_2, \ldots, \D_N$ that vary from episode to episode, i.e. $f_{t, i} = f_i \sim \D_i$.
% Let $\E$ denote the unconditional expectation with respect to the randomness of the functions, and $\bar{x}_i = \frac{1}{T}\sum_{t=1}^T x_{t, i}$ be the average iterate in episode $i$. Let $\bar{J}_i(\A) = \frac{1}{T} J_i(\A)$, Theorem \ref{thm:main_intro} implies
% \begin{align*}
% \frac{1}{N}\sum_{i=1}^N \E[f_i(\bar{x}_i)]
% &\le \min_{\A^*\in \Pi} \frac{1}{N}  \sum_{i=1}^N\E\left[\bar{J}_i(\A^*)\right]+ \tilde{O}\left(\frac{1}{\sqrt{TN}}\right).
% \end{align*}
% Thus, our algorithm guarantees that the average expected function value over $N$ episodes is close to the average expected cost of the best algorithm in hindsight.  

Suppose our functions are drawn randomly from distributions $\D_1, \D_2, \ldots, \D_N$ that vary from epoch to epoch, i.e. $f_{t, i} \sim \D_i$. 
Let $\E$ denote the unconditional expectation with respect to the randomness of the functions, and define the function $\bar{f}_i(x) := \E_{\D_i}[f_{t, i}(x)]$, then Theorem \ref{thm:main_intro} implies
\begin{align*}
\frac{1}{NT}\sum_{i=1}^N\sum_{t=1}^T \E[\bar{f}_i(x_{t, i})]
&\le \min_{\A^*\in \Pi}\E\left[\bar{J}(\A^*)\right]+ \tilde{O}\left (\frac{1}{\sqrt{TN}}\right ).
\end{align*}

Thus, our algorithm guarantees that the average expected function value is close to the average expected cost of the best algorithm in hindsight. Note that since the functions are changing, $\Pi$ has a more subtle definition in the stochastic and adversarial settings (see Section \ref{sec:thm+alg} for details).

\paragraph{Adversarial optimization.}
In the adversarial setting, our functions $f_{t, i}$ arrive in an online manner at each time step, and the standard optimization metric is regret. Our main theorem gives a guarantee over the per-episode average regret of our optimization algorithm. Recall the definition of regret in an episode: $\regret_i = \sum_{t=1}^T f_{t, i}(x_{t, i}) - \min_{x_i^*} \sum_{t=1}^T f_{t, i}(x_i^*)$. Our algorithm satisfies 
\begin{align*}
\frac{1}{TN}\sum_{i=1}^N \text{Regret}_i &\le \min_{\A^*\in \Pi}\frac{1}{TN}\sum_{i=1}^N\text{Regret}_i(\A^*) + \tilde{O}\left (\frac{1}{\sqrt{NT}}\right ).
\end{align*}

% \eh{table of results}

\subsection{Related work.}

\subsubsection*{Online convex optimization and nonstochastic control.}

Our methods for meta-optimization are based on iterative gradient methods with provable regret guarantees. These have been developed in the context of the online convex optimization framework. In this framework, a decision maker iteratively chooses a point $x_t \in \K$, for a convex set in Euclidean space $\K \subseteq \reals^d$, and receives loss according to a convex loss function $f_t:\K \mapsto \reals$. The goal of the decision maker is to minimize her regret, defined as
$$ \regret = \sum_{t=1}^T f_t(x_t) - \min_{x \in \K} \sum_{t=1}^T f_t(x) . $$
Notably, the best point $x^* \in \K$ is defined only in hindsight, since the $f_t$'s are unknown a priori. For more information on this setting and an algorithmic treatment see \citep{hazan2016introduction}.

Techniques from online convex optimization were instrumental in developing an online control theory that permits nonstochastic disturbances and is regret-based. Deviating from classical control theory, the online nonstochastic control framework treats control as an interactive optimization problem, where the objective is to make decisions that compete with the best controller in hindsight. The book of \citet{hazan2022introduction} gives a comprehensive survey of the topic.

In general, computing the best controller in hindsight is a nonconvex optimization problem, even for linear dynamical systems and linear controllers, if we allow general convex costs. However, online nonstochastic control algorithms have provable guarantees in this setting despite the challenge of nonconvexity, since they employ convex relaxation techniques by executing policies in a larger policy class. 

%This is crucial for our guarantees: meta-optimization can be a nonconvex problem, but we can hope to use feedback control to obtain meta-regret bounds.

\subsubsection*{Average case analysis of optimization.} A closely related framework for analyzing optimization methods is the average-case analysis framework developed in \citep{average_case}. This framework studies the expected performance of optimization algorithms when the problem is sampled from a distribution, and allows for more fine-grained results than typical worst-case analysis. Average-case optimal first-order methods for minimizing quadratic objectives are proposed in \citep{average_case}, and  \citet{domingo-enrich2021averagecase} extend the study to bilinear games. %\cite{polyak_universality} further show that any average-case optimal first order method converges to polyak momentum, shedding new light on understanding its empirical success. 
In the stochastic setting, average-case analysis of SGD is given by \citet{average_case_sgd}.

Compared to the average-case analysis framework, meta-optimization is significantly more general, since we do not assume known stochastic distribution of the optimization problems, and we compete with the best algorithm in hindsight. In contrast, implementing the optimal algorithms under the average-case framework requires knowledge of the problem distribution. 

% \eh{add more citations/relevant papers in this line of work}

\subsubsection*{Hypergradient descent and hyperparameter tuning for optimizer parameters.}
Hyperparameter optimization is a significant challenge in the practice of deep learning, and has been intensively studied. The work of \citet{baydin2017online} apply local iterative methods to the problem of optimizing the learning rate from an empirical standpoint, and \citet{chandra2019gradient} give better practical methods. However, even the simple case of optimizing the learning rate can be nonconvex, and it is unlikely that iterative methods will yield global optimality guarantees. Certain provable guarantees for quadratic functions and scalar learning rate are presented in \citep{wang2021guarantees}.

\noindent More general hyperparameter optimization techniques were also applied to the same problem, most commonly Bayesian optimization \citep{snoek2012practical} and spectral techniques \citep{hazan2017hyperparameter}.

\subsubsection*{Performance estimation programming.} The Performance Estimation Problem (PEP) was first proposed in \cite{pep}, and can be seen as the worst-case optimal approach of learning the best optimizer for a class of functions. The PEP can be formulated as a maximization problem, and though it is nonconvex, SDP relaxations are introduced in \cite{pep}.
%and new performance bounds on first-order methods were obtained. 
\citet{pep_smooth} further propose convex programs that can find the exact worst-case performance of first order methods on smooth convex functions. More recently, \citet{bnb_pep} present BnB-PEP, a PEP framework that extends to nonconvex optimization. The problem of finding the optimal method is formulated as a nonconvex but practically tractable QCQP, and algorithms are given to solve the problem to global optimality.

In contrast to the PEP framework, meta-optimization is an online and sequential formulation of learning the best optimizer. 
% Since the objective functions can be adversarial, the best optimizer is determined only in hindsight. 
Moreover, our notion of regret is instance-optimal: we find the best optimizer for the objective functions that appear in the meta-optimization problem, instead of over the entire function class. The benchmark algorithm class we consider also differ from first-order methods studied in PEP, as we allow preconditioned methods that can potentially adapt to the geometry of the problem.

\subsubsection*{Control for optimization.} 
The connections between control and optimization go back to Lyapunov's work and its application to the design and analysis of optimization algorithms. We survey the various approaches in detail in  Section  \ref{sec:opt_and_control}.

\citet{iqc} apply control theory to the analysis of optimization algorithms on a single problem instance. They give a general framework, using the notion of Integral Quadratic Constraints from robust control theory, for obtaining convergence guarantees for a variety of gradient-based methods. This framework can also be used to design algorithms given target performance characteristics. \citet{pmlr-v134-casgrain21a} study the characterization of the regret-optimal algorithm given an objective function, using a value function-based approach motivated by optimal control. The goal is to characterize the regret-optimal algorithm for a particular optimization instance, instead of developing provably efficient methods.

\subsubsection*{Adaptive gradient methods.} In contrast to learning the optimal algorithm, the methodology of adaptive preconditioning aims to compete with the best regularizer from a class using online learning. Adaptive algorithms starting from Adagrad \cite{duchi2011adaptive}, followed by RMSprop and Adam \citep{tieleman2012lecture,ADAM}, give principled methods for auto-tuning the preconditioning matrix for gradient descent and its variants. Yet in practice, they require a multiplicative learning rate factor that needs tuning. 

% Further, the scope of adaptive preconditioning is limited and does not capture other hyperparameters such as momentum and weight decay. 
%More generally, hyperparameter tuning, including tuning the preconditioning matrix, is an instance of the meta-optimization problem of learning the best optimization algorithm within an algorithm class. 
\subsection{Organization.}
In Section \ref{sec:opt_and_control}, we give an overview on prior works using control theory in optimization. We describe the recent framework of online nonstochastic control in Section \ref{sec:nonstochastic}, and why it is important for meta-optimization. In Section \ref{sec:dynamics}, we introduce the new control formulation of meta-optimization, and give some examples of optimizers that can be expressed as control policies. In Section \ref{sec:thm+alg}, we state the algorithm and main results for convex quadratic meta-optimization, as well as the benchmark algorithm class. We also give an illustrative example of learning the learning rate for convex quadratics. Then, we extend our results to convex smooth meta-optimization in Section \ref{sec:smooth}. We provide the analysis of our main results, including the technical derivation of control with unbounded disturbances in Section \ref{app:main_thm_proof}.

%% file: opt_and_control.tex
\section{Mathematical optimization and feedback control.}\label{sec:opt_and_control}

The fields of mathematical optimization and control theory are closely related. Many optimization methods are inspired by natural dynamical systems, such as Polyak's heavy ball method \citep{polyak1964some}.  On the other hand, the analysis of optimization algorithms also has fundamental connections with the mathematics of dynamical systems. Besides the heavy ball method\citep{wang2021modular}, other examples where optimization algorithms are analyzed as dynamical systems include Nesterov momentum for smooth functions \citep{su2014differential,muehlebach2019dynamical}, and more recently frameworks of analysis using Lyapunov stability theory \citep{iqc, wilson2018lyapunov}. We briefly describe these prior connections in the next two subsections. 

A {\bf dynamical system} is a vector field mapping $\reals^d$ onto itself. It can be written as 
$$ z_{t+1} = v(z_t) , $$
where $v$ is the dynamics function. We use discrete time notation throughout this paper as optimization methods implemented on a computer admit discrete-time representations. Dynamical systems can be used to describe an optimization algorithm; for example, gradient descent for an objective $f$ with learning rate $\eta$ can be written as 
$ x_{t+1} = x_t - \eta \nabla f(x_t) ,$
and similarly, other iterative preconditioned gradient (or higher-order) method can be described as a dynamical system.  
The natural question of convergence to local or global minima, as well as the rate of convergence, can be framed as a question about the {\it stability} of the dynamical system.  Informally, a  dynamical system is said to be stable from a given starting point $x_0$ if the dynamics converges to an equilibrium from this point. There are numerous definitions of stability and equilibria of dynamical systems, and we refer the interested reader to comprehensive discussions in \citep{slotine1991applied,hazan2022introduction}. In this introductory section we consider only the most intuitive notion of convergence to a global minimum for a convex function, and the basic setting of noiseless dynamical systems in a single trajectory.

In his foundational work, \citet{lyapunov1992general} introduced two methods for certifying stability of dynamical systems, the direct method and the indirect method.

\subsection{Lyapunov's direct method.} 

Lyapunov's direct method is by far the most widely used framework of control in optimization. It centers on creating an energy or potential function, called the ``Lyapunov function", which needs to be non-increasing along the trajectory of the dynamics, and strictly positive except at the equilibrium (global minimum for an optimization problem) to certify stability. 
A common example given in introductory courses on dynamics is that of the motion equations of the pendulum. The Lyapunov function for this system is taken to be the total energy, kinetic and potential \citep{tedrake}.

For the discrete dynamics of gradient descent over a strongly convex objective $f$, the standard Lyapunov function to use is simply the Euclidean distance to optimality, or $\mathcal{E}(x) = \frac{1}{2}\|x - x^*\|^2 $, where $x^*$ is the global minimizer. It can be shown that with a sufficiently small learning rate depending on the strong convexity parameter, that this energy function is monotone decreasing for the dynamics of gradient descent, showing that the system is stable \citep{wilson2018lyapunov}. Various other conditions on the objective function $f$, such as smoothness, convexity, and so forth, give rise to different energy functions that can certify stability, and even show rates of convergence. 

However, Lyapunov’s direct method cannot be used to provide optimality guarantees for control: prescriptive suggestions for how to drive the system to a target state. Moreover, standard Lyapunov analysis does not take disturbances of the system into account, even though they naturally arise in most physical systems\footnote{An exception is the work of \citet{iqc} .}.  To overcome these limitations, we consider Lyapunov's second technique, the indirect method.

\subsection{Lyupanov's indirect method.}

Instead of working directly with the nonlinear system, this method studies the behavior of the system around the linearization about the equilibrium point. More formally, let $z_0,...,z_T$ be a given trajectory, then we can approximate the dynamics as 
$$ z_{t+1} = A_t z_t +  w_t , $$ 
where $A_t$ is the Jacobian of the dynamics with respect to $z_t$ and $w_t$ is a noise term that can model misspecification or other disturbance. The stability of the linearized system can be determined by inspecting the spectrum of the dynamics matrices $A_t$.

On one hand, this technique is more limited: the analysis only applies to linear dynamical systems. On the other hand, it is also amenable to optimal control theory. This additional power can potentially guarantee convergence to the optimal method in optimization, which is exactly our goal. Nevertheless, there are several shortcomings of this approach, especially when applied to optimization, including:
\begin{enumerate}
    \item The linearization depends on the state: $A_t$ is a function of $x_t$, and optimizing the trajectory over the sequence of linearized dynamics does not imply global optimality for the original system. 

\item The linearization is a faithful approximation of the dynamical system only if the dynamics is smooth, and the time interval between measurements is small with respect to this smoothness. 
    
\end{enumerate}

These limitations might explain, at least partially, why Lyapunov's indirect method has not been used to analyze optimization algorithms. However, with this method we can incorporate a control signal, as well as a disturbance, into the nonlinear dynamics formulation,
$ z_{t+1} = v(z_t,u_t) + w_t ,$
where $u_t$ can capture parameters of the optimization method, such as the learning rate and preconditioner. Using Lyapunov's indirect method we can write
$$ z_{t+1} = A_t z_t + B_t u_t + w_t , $$
where $A_t,B_t$ are the Jacobians with respect to the state and the control. We now have a linear time-varying (LTV) dynamical system, and if the objective functions are quadratic and the disturbance is stochastic, then the optimal controller can be computed using LQR theory \citep{kalman1960new}.

This observation is the starting point of our investigation. Numerous challenges arise when we attempt to use this methodology to learn the optimizer: 
\begin{enumerate}
    \item Optimal control theory requires the knowledge of system dynamics a priori, but in optimization they are only determined during the optimization process. 

    \item Efficient algorithms for optimal control, based on the Bellman equation and backward induction, are restricted to quadratic cost functions. 

    \item Optimal control requires the disturbance to be stochastic, and it is not robust to adversarially chosen cost functions that arise in online or stochastic optimization. 

\end{enumerate}

By using new techniques in control theory developed in the context of machine learning, namely nonstochastic control, we show how to overcome these challenges:
\begin{enumerate}
    \item Online nonstochastic control does not require the knowledge of system matrices a priori. Further, it allows adversarially chosen systems and cost functions. 

    \item The methods for meta-optimization we consider are themselves gradient-based and scalable. Thus we can hope to devise practical algorithms when the number of episodes increases, or the problem dimension is high. 

    \item Online nonstochastic control methods have strong regret guarantees under adversarially changing cost functions, which can be extended to obtain finite-time provable regret bounds in meta-optimization. 

\end{enumerate}

We describe the framework of online nonstochastic control in the next subsection. 
\subsection{Online nonstochastic control.}\label{sec:nonstochastic}

The online nonstochastic control (ONC) framework applies online convex optimization to new parametrizations of classical control problems. This section gives the basic description of this framework, and a detailed exposition appears in   \citep{hazan2022introduction}. 

\subsubsection*{Problem setting.}Consider first the simple case of a linear time invariant (LTI) dynamical system in a single trajectory without resets. 
A linear dynamical system (LDS) evolves via the following equation:
$$ z_{t+1} = Az_t + Bu_t + w_t.$$

Here $z_t \in \reals^{d_z}$ represents the state of the system, $u_t \in \reals^{d_u} $ represents a control input and $ w_t \in \reals^{d_x} $ is a disturbance introduced to the system from the environment.  The goal of the controller is to produce a sequence of control actions $ u_1 \ldots u_T$ aimed at minimizing the cumulative control cost $ \sum_{t=1}^T c_t(z_t, u_t)$. 
Many systems do not exhibit full observation, and a well-studied model for capturing partial observation is when the observation is a linear projection of the state, i.e. 
$  y_t = C x_t + D u_t + \xi_t , $
where $y_t \in \reals^{d_y}$ is the observation at time $t$ and $\xi_t \in \reals^{d_y}$ is an additional noise term that affects the observed signal. We say that a system is fully observed if $z_t$ is observed by the controller, and usually refer to this case unless specifically stated otherwise.

The control inputs, when correctly chosen, can modify the system to induce a particular desired behavior. For example, controlling the thermostat in a data center to achieve a certain temperature, applying a force to a pendulum to keep it upright, or driving a drone to a destination.

In nonstochastic control, we instead consider a significantly broader class of general (possibly non-quadratic) convex cost functions and norm-bounded (instead of stochastic) disturbances. Both the costs and disturbances may be adversarially chosen, and only be revealed to the controller in an online fashion.

\subsubsection*{A new objective: policy regret.} This new objective builds upon the theory of online convex optimization \citep{hazan2016introduction} and regret minimization in games: instead of computing the optimal policy in a certain class, we can compete with it using improper learning via convex relaxation of the policy class. Formally, we measure the performance of a policy through the notion of policy regret,

\begin{equation} \label{eqn:policyregret}
 \mathrm{Regret} = \sum_{t=1}^T c_t(z_t, u_t) - \min_{\pi \in \Pi } \sum_{t=1}^T c_t(z_t^\pi, u_t^\pi),
\end{equation}
where $ z_t^\pi$ represents the state reached when executing the policy $\pi$. In particular, the second term represents the total cost paid by the best (in hindsight) policy from the class $\Pi$ had we played it under the same sequence of disturbances and cost functions. In this regard, the above notion of regret is counterfactual and hence more challenging than the standard stateless notion of regret.
Algorithms which achieve low policy regret are naturally adaptive, as they can perform almost as well as the best policy in the long run, even when the disturbances and costs are adversarial.

But what policies are reasonable to compare against? We survey the state-of-the-art in control policies next. Then we describe new methods arising from this theory that can provably compete with the strongest policy class.

\subsubsection{Existing and new policy classes for control.}

\paragraph{Linear state-feedback policies.} 
For a matrix $K \in \reals^{d_u\times d_z}$, we say a policy of the form $u_t = Kz_t$ is a linear state-feedback policy, or linear policy. In classical optimal control with full observation, the cost function are quadratic in the state and control.
Under this assumption, if the system is LTI with stochastic disturbances, then the infinite-horizon optimal policy can be computed using the Bellman optimality equations (see e.g. \citep{tedrake}). This gives rise to the Discrete-time Algebraic Riccati Equation (DARE), whose solution is the optimal policy, and it is linear.
The finite-horizon optimal policy can also be derived and shown to be linear. It is thus reasonable to consider the class of all linear policies as a comparator class, especially for LTI dynamical systems. Denote the class of all stabilizing linear policies as
$$ \Pi_{Lin} = \{ K \in \reals^{d_z \times d_u } \} . $$

% \paragraph{Linear dynamical controllers. }

% A generalization of static state-feedback policies is that of linear dynamical controllers (LDCs). LDCs are particularly useful for partially observed LDS and maintain their own internal dynamical system according to the observations, in order to recover the hidden state of the system.  

% \begin{definition}[Linear Dynamic Controller] \label{def:ldc}
% A linear dynamic controller $\pi$ is a linear dynamical system $(A_\pi, B_\pi, C_\pi, D_\pi)$ with internal state $s_t\in \mathbb{R}^{d_\pi}$, input $z_t\in \mathbb{R}^{d_z}$ and output $u_t\in\mathbb{R}^{d_u}$ that satisfies
% $$
% s_{t+1} = A_\pi s_t + B_\pi z_t,\ \ u_t = C_\pi s_t + D_\pi z_t.
% $$
% \end{definition}

% LDCs are state-of-the-art in terms of performance and prevalence in control applications of LDS, both in the full and partial observation settings. They are known to be theoretically optimal for partially observed LDS with quadratic cost functions and normally distributed disturbances, but are more widely used in practice. For a formal definition, see Appendix \ref{app:nonstochastic}.

\paragraph{Linear dynamical control policies. }

A generalization of static state-feedback policies is that of linear dynamical controllers (LDCs). LDCs are particularly useful for partially observed LDS and maintain their own internal dynamical system according to the observations, in order to recover the hidden state of the system.  A formal definition is given below. 

\begin{definition}[Linear Dynamical Controller] \label{def:ldc}
A linear dynamical controller $\pi$ is a linear dynamical system $(A_\pi, B_\pi, C_\pi, D_\pi)$ with internal state $s_t\in \mathbb{R}^{d_\pi}$, input $z_t\in \mathbb{R}^{d_z}$ and output $u_t\in\mathbb{R}^{d_u}$ that satisfies
$$
s_{t+1} = A_\pi s_t + B_\pi z_t,\ \ u_t = C_\pi s_t + D_\pi z_t.
$$
\end{definition}
LDCs are state-of-the-art in terms of performance and prevalence in control applications of LDS, both in the full and partial observation settings. They are known to be theoretically optimal for partially observed LDS with quadratic cost functions and Gaussian disturbances, but are more widely used in practice. Denote the class of all stabilizing LDCs as
$$ \Pi_{LDC} = \{ A \in \reals^{d_s \times d_s} ,B \in \reals^{d_s \times d_z}, C \in \reals ^{d_u \times d_s} ,D \in \reals^{d_u \times d_z } \} . $$

\paragraph{Disturbance-feedback controllers.}

An even more general class of policies is that of disturbance-feedback controllers (DFCs), where the policies are functions of past disturbances.
\begin{definition}[Disturbance-feedback controller]\label{def:dac_gen}
A disturbance-feedback controller with parameters $(K, M)$, where $M = [M^1, \ldots, M^L]$, outputs control $u_t$ at state $z_t$,

$$ u_t = Kz_t + \sum_{i=1}^L M^i w_{t-i}. $$
Here $M^i$ denotes the $i$-th matrix in $M$, instead of a matrix to its $i$-th power.
\end{definition}

Denote the class of DFCs as 
$$ \Pi_{DFC} = \{ K \in \reals^{d_s \times d_s} ,M\in \reals^{d_u\times Ld_s}: K\text { is stabilizing } \}. $$

This policy class is more general than that of LDCs and linear controllers, because it can approximate the latter classes in terms of the average control cost. Broadly speaking, for every time-invariant LDS and every stabilizing policy in $\Pi_{LDC}$ and $\Pi_{Lin}$, there exists a DFC with $L=\Omega(\log\frac{1}{\eps})$ whose average cost is $\eps$-close on the same system and sequence of disturbances \cite{hazan2022introduction}. 
We henceforth study regret with respect to the class of DFCs, which is the most powerful of the above policy classes and gives the strongest performance guarantees.

% \paragraph{Why is regret against DFCs meaningful?}  

% Since the class of DFCs is more general than $\Pi_{LDC}$, competing with the best DFC translates to competing with the best LDC. The latter is know to be the optimal policy for partially observed LDS with zero-mean Gaussian disturbances, a fundamental problem for control theory known as the LQG  (see e.g. \citep{simchowitz2020improper}).  %Furthermore, LDC controllers are considered SOTA and widely used even in full information LTI LDS  control. 

% It follows that sublinear regret against LDCs implies near-optimality in these widely-studied theoretical settings. In addition, this guarantee is meaningful even for adversarial  disturbances and general convex cost functions. Indeed, no explicit form of the optimal policy is known for general convex cost functions, and it is conjectured to be intractable by \citet{rockafellar1987linear}.

\subsubsection{The gradient perturbation controller.}

The fundamental new technique introduced in \citep{agarwal2019online} is a novel algorithm called the Gradient Perturbation Controller (GPC) for the nonstochastic control problem. For simplicity, assume that the dynamical system given by $(A, B)$ is known and the state is fully observable.  Thus, given a sequence of controls and states, we can compute the corresponding sequence of disturbances.

It can be shown that directly learning the optimal linear controller $K^*$ is not a convex problem. However, instead of learning $K^*$, we can learn a sequence of matrices $\{ M^i \}_{i=1}^t$ (where $i$ denote the index instead of matrix power) that represents the effect of $ w_{t-i}$ on $ u_t$ under the execution of some linear policy. Schematically, we parameterize our policy as a DFC,
$ u_t =  Kx_t + \sum_{i=1}^t M^i w_{t-i},$ where $K$ is a stabilizing controller of the system $(A, B)$. 
Since the states $x_t$'s are linear in the past controls, and the choice of the controller ensures that the controls are linear in $\{M^i\}_{i=1}^t$, the states are also linear in $\{M^i\}_{i=1}^t$. Moreover, since the cost functions are convex in the states and controls, they are convex in $\{M^i\}_{i=1}^t$, the parameters of interest. We can thus hope to learn the parameters using standard techniques such as gradient descent and Online Newton Step.

However, there are two challenges with this approach. First, the number of parameters grows linearly with time,  and so can the regret. Second,
The decision a controller makes at a particular instance affects the future through the state.

To resolve the first  issue, we limit the  history length $L$ of the GPC to grow very slowly with time. It can be shown that for stable (and stabilizable given a stabilizing controller) systems, a history length of $O(\log \frac{1}{\epsilon})$ is sufficient to capture the class of infinite-memory DFCs up to an $\eps$ additive approximation in terms of the average cost. This logarithmic dependence of the history length on the approximation guarantee means that for the number of parameters to grow mildly, the policy regret is affected by no more than a constant factor.

% This logarithmic dependence of the history length on the approximation guarantee means that the policy regret is not affected by more than a constant factor, if we choose $\eps $ to be  inversely polynomial in $T$, while the number of parameters grows mildly.

The second issue is more subtle. Fortunately, online learning of loss functions with memory was a topic studied before in   \citep{anava2015online}. It is shown that gradient methods guarantee near-optimal regret if the learning rate is tuned as a function of the memory length.

With all the core components in place, we provide a brief specification of the GPC algorithm in Algorithm \ref{alg:mainA}. The GPC algorithm accepts as input a stabilizing controller $K$ that ensures 
$ \rho(A + BK) < 1 ,$
where $\rho$ is the spectral radius of the matrix $A+BK$. Such a controller $K$ can be computed for all stabilizable systems using semi-definite programming, see e.g. \citep{cohen2019learning}. The algorithm then proceeds to control using a DFC policy that is adapted to the online cost functions. 

To obtain a loss function of the parameters of the GPC, we compute the \textit{surrogate state} $y_t^M$ and \textit{surrogate control} $v_t^M$: the terminal state and control we would have seen by executing the current GPC for $L$ time steps from the zero state, under the same sequence of disturbances. We then compute the \textit{surrogate loss}, $c_t(y_t^M, v_t^M)$, which is a function of $M=\{M^l\}_{l=1}^L$. Notice that the gradient of the cost function $g_t$ is taken with respect to the policy variables $M_t=M_t^{1:L}$. This is valid since both the control, and in turn the state, are a convex function of these variables. The notation $\prod_\M(x)$ denotes the Euclidean projection of a vector $x$ onto the set $\M$, see \citep{hazan2016introduction} for more details on projections.

\begin{algorithm}[h]
\caption{Gradient perturbation controller} 
\label{alg:mainA}
\begin{algorithmic}[1]
\State \textbf{Input:} Step size schedule $\eta_t$ , history length $L$, constraint set $\M$, stabilizing controller $K$.
\State Initialize $M_1^{1:L} \in \M$ arbitrarily.
\For{$t = 1, \ldots, T$}
\State Choose the action: $u_t = K z_t+ \sum_{l=1}^{L} M_t^{l} w_{t-l}.$
\State Observe the new state $z_{t+1}$, cost function $c_t$, and record $w_t=z_{t+1}-Az_t-Bu_t$.
\State Compute $y_t^{M_t}, v_t^{M_t}$: from $y_{t-L}^{M_t}=0$, execute the DFC with parameters $M_t$ for $L$ time steps, record the terminal state and control.
\State Obtain surrogate loss $g_t(M_t) = c_t(y_t^{M_t}, v_t^{M_t})$.
\State Online gradient update:
\[M^{t+1} = \Pi_{\M}(M^t - \eta_t \nabla g_{t}(M_t)) . \]
\EndFor
\end{algorithmic}
\end{algorithm}

The GPC algorithm is accompanied by a near-optimal regret guarantee against the class of DFC policies with history length $L$, 
\begin{theorem}[Theorem 5.1 in \citep{agarwal2019online}]
Let $u_t$ be a sequence of controls generated by Algorithm \ref{alg:mainA} for a known LDS, and let $\Pi_{DFC}^L$ denote the set of DFCs with history length $L$. Then for any arbitrary bounded disturbance sequence and convex cost functions, it holds that 
$$ \sum_{t=1}^T c_t(z_t,u_t)  - \min_{\pi \in \Pi_{DFC}^L} \sum_{t=1}^T c_t(z_t^\pi , u_t^\pi) \le \tilde{O}(\sqrt{T} ) . $$
\end{theorem}

\ignore{

\subsubsection{Significance of online nonstochastic control to meta-optimization} \label{sec:qna}
\paragraph{Significance of online nonstochastic control for online linear control.} 
Online nonstochastic control yields the first provably efficient method for LQR under general cost functions and adversarial disturbances with finite-time guarantees, which was an important open problem proposed by  \cite{rockafellar1987linear}.  The LQR is a fundamental problem and a building block for more sophisticated techniques in model predictive control.

\paragraph{Significance with respect to meta optimization. }

We formulate meta-optimization as a control problem with adversarial disturbances. Furthermore, our goal is to minimize meta-regret which is an adversarial notion with respect to the best policy in hindsight.

Regret minimization with nonstochastic disturbances was not known before the introduction of the online nonstochastic control framework, and the new methods therein. In formulating meta-optimization as a control problem, we make use of the crucial fact that these methods can tolerate adversarial disturbances, which are chosen online along a path of optimization.
}

%% file: dynamical_system.tex
\section{The dynamics of meta-optimization.}\label{sec:dynamics}

In this section, we introduce our dynamical systems formulation for the meta-optimization problem. For simplicity, consider the online setting where all functions $f_{t, i}$ are convex quadratic. Denote the location-independent Hessian as
$ \nabla^2 f_{t, i} (x)  = H_{t, i} \in \reals^{d\times d}. $
This setting can be generalized to convex smooth functions, as we elaborate in Section \ref{sec:smooth}.  We assume that the functions have bounded Hessian, and the gradients satisfy:
\begin{assumption}\label{assumption:gradient} There exists a smoothness parameter $\beta \ge 0$ and $b \ge 0$ such that
$ \forall i\in [N],\ t\in [T],\ x\in \reals^d,\ \ 
\|\nabla f_{t, i}(x)\| \le 2\beta\|x\|+b.$
\end{assumption} 
 % \eh{add more histories, and explain what kind of algorithms the best linear controller contains, in det/sto/adversarial settings}

We first present the most general formulation of the dynamical system. For any given $h \le T/2$, consider  the time-variant, discrete linear dynamical system of dimension $2hd$ as follows

\begin{align*}
 z_{t+1, i} = \begin{bmatrix}
x_{t+1, i}\\
\vdots \\
x_{t-h+2, i} \\ 
\nabla f_{t, i}(x_{t, i}) \\
\vdots \\
\nabla f_{t-h+1, i}(x_{t-h+1, i})
\end{bmatrix} 
&= 
A_{t, i}\begin{bmatrix}
x_{t, i} \\
\vdots \\
x_{t-h+1, i} \\
\nabla f_{t-1, i}(x_{t-1, i}) \\
\vdots \\
\nabla f_{t-h, i}(x_{t-h, i}) 
\end{bmatrix} 
+
B u_{t, i}  
+ 
\begin{bmatrix}
0 \\
\vdots \\
0 \\
\nabla f_{t, i}(x_{t-1, i})\\
0\\
\vdots \\
0
\end{bmatrix}  
,\text{ where }
\end{align*}
\begin{equation}\label{def:dynamics}
A_{t, i} = \begin{bmatrix}
(1-\delta)I & 0 &\cdots & \cdots & 0 & -\eta I_d & 0 & \cdots &0\\
I_{d(h-1)} &  &  & & 0 & 0 &\cdots &\cdots& 0\\
H_{t, i}& -H_{t, i} & 0 &\cdots & 0 & 0 & 0 & \cdots &0\\
0 & \cdots &\cdots &\cdots & 0 & I_{d(h-1)} & & & 0
\end{bmatrix},\ \ 
B = \begin{bmatrix}
I_d & 0 & \cdots & 0\\
0 & 0 & \cdots & 0\\
&\vdots&\\
0 & 0 & \cdots & 0
\end{bmatrix}.
\end{equation}

The dynamical system given in Section \ref{subsec:contribution} is a special case of the general formulation above. For completeness we restate the evolution below,
\begin{equation}\label{eq:dynamics_special}
\begin{bmatrix}
x_{t+1, i} \\
x_{t, i}\\
\nabla f_{t, i}(x_{t, i})
\end{bmatrix} 
= 
\begin{bmatrix}
(1-\delta)I & 0   &  -\eta I\\
I & 0 & 0 \\
 H_{t, i} & -H_{t, i} & 0
\end{bmatrix} \begin{bmatrix}
x_{t, i} \\
x_{t-1, i}\\
\nabla f_{t-1, i}(x_{t-1, i})
\end{bmatrix} 
+ 
\begin{bmatrix}
I & 0 & 0 \\
0 & 0 & 0\\
0 & 0 & 0
\end{bmatrix}  u_{t, i} 
 + \begin{bmatrix}
 0\\
 0\\
 \nabla f_{t, i}(x_{t-1, i})
 \end{bmatrix}.
 \end{equation}
Both systems above are valid representations of the optimization process given that
$$
\nabla f_{t, i}(x_{t, i}) = \nabla f_{t, i}(x_{t-1, i}) + H_{t, i}(x_{t, i} - x_{t-1, i}).$$

The states in the general formulation consist of the last $h$ iterates and time-delayed gradients. As we show in Section \ref{sec:thm+alg}, in some settings, the more expressive the states, the richer the benchmark algorithm class. Similar to (\ref{eq:dynamics_special}), the control-free trajectory of the general formulation describes time-delayed gradient descent with weight decay. For simplicity, the main results in this paper use the formulation given by (\ref{eq:dynamics_special}), and extension to the general formulation is left as future work.

\subsection*{Resets.} The dynamical system evolves according to (\ref{eq:dynamics_special}) during an episode. However, each episode starts with the optimization iterate at an arbitrary initialization $x_{1, i}$ and we need to reset the system state before a new episode begins. To transition the system state to the new initialization, consider the following reset disturbance,
\begin{equation}\label{eq:reset}
w_{T, i} = \begin{bmatrix}
x_{1, i+1} - ((1-\delta)x_{T, i} - \eta \nabla f_{T-1, i} 
(x_{T-1, i}) + \bar{u}_{T, i})\\
x_{1, i+1} - x_{T, i}\\
% - H_{T, i}(x_{T, i} - x_{T-1, i})
\nabla f_{T, i}(x_{T-1, i}) - \nabla f_{T, i}(x_{T, i})
\end{bmatrix},\end{equation}
where $\bar{u}_{T, i}$ is the top $d$ entries of the control signal $u_{T, i}$. Under the reset disturbance, the initial state of an epoch is
$
z_{1, i} = \begin{bmatrix}
x_{1, i}^\top &
x_{1, i}^\top &
0&
\end{bmatrix}^\top,$ consistent with the meta-optimization protocol. Finally, we assume that the initializations have bounded norm.
\begin{assumption}\label{assumption:reset}
For all $i$, $x_{1, i}$ satisfies $\|x_{1, i}\| \le R$.
\end{assumption}

\subsection*{Cost functions.} For convex optimization, we consider minimizing the objective function value. In contrast to previous works that analyze other objectives, such as distance to optimality, we choose function value because it is completely known to the algorithm designer. The corresponding cost function is $c_{t,i}(z_{t, i}, u_{t, i}) = f_{t, i}(Sz_{t, i}) = f_{t, i}(x_{t, i})$, for a matrix $S$ that selects the first $d$ entries of $z_{t, i}$.

% We remark that the given dynamical system formulation cannot be straightforwardly applied to minimizing the gradient norm, which is the standard objective in nonconvex optimization. We leave it as future work to resolve this issue. 

\subsection*{Stability.} Stability is a central concept in optimal control. In nonstochastic control, a stable system or access to a stabilizing controller is necessary to achieve regret that is polynomial in the state dimension \citep{black_box}. The notion of stability in this context is more refined than the conventional notion of bounded spectral radius. For LTI systems, it is referred to as strong stability, which requires the dynamics to be stable with a margin under some transformation. 

\begin{definition}[Strong stability]\label{def:strong_stability}
A system is $(\kappa, \gamma)$ strongly stable if there exist matrices $P, Q$, such that $A = PQP^{-1}$, and $\|Q\| \le 1-\gamma$, $\|P\|,\|P^{-1}\| \le \kappa$.
\end{definition}
In other words, strong stability ensures that the state decays exponentially fast under the evolution of the system. Though this condition seems stringent, the following lemma shows that if a system is stable, it is also strongly stable for some $(\kappa, \gamma)$. 
\begin{lemma}[Lemma B.3 in \cite{tomer}]\label{lem:tomer} If the system $A$ is stable with $\rho(A) < 1-\gamma$, it is also $(\kappa, \gamma)$-strongly stable, where $\kappa = \max\{\|P\|,\|P^{-1}\|\}$ with $
P = \sum_{i=0}^\infty (A^i)^\top A^i.$
\end{lemma}
We show that with a proper choice of $\eta$, our dynamical system is stable even without control inputs, and is therefore strongly stable. The following lemma shows that stability holds for (\ref{eq:dynamics_special}) with $\eta$ scaling inversely as the smoothness of the function, and $\delta$ arbitrarily small. This restriction on $\eta$ is natural, since gradient descent with learning rate larger then $\frac{1}{\beta}$ diverges. 

\ignore{
\begin{lemma}\label{lem:stability_general_system}
Suppose $0 \preceq H_t \preceq \beta I$, then for $\delta\in(0, \frac{1}{2}], \eta < \frac{\delta}{16\beta}, $ for $A_{t, i}$ defined in \ref{def:dynamics}, $$
\rho(A_{t, i}) < 1-\frac{\delta}{2} < 1.$$
\end{lemma} 

The formulation given in Section \ref{subsec:contribution}, a special case of the general formulation, allows for choices of $\eta$ independent of $\delta$.
}

\begin{lemma}\label{lem:stability}
Suppose $0 \preceq H_{t, i} \preceq \beta I$, then for $\eta \le \frac{1}{8\beta}, \delta\in(0, \frac{1}{2}]$, we have 
$$\rho\left(\begin{bmatrix}
(1-\delta)I   &  0 & -\eta I\\
I & 0 & 0\\
H_{t, i} & -H_{t, i} & 0
\end{bmatrix} \right) < 1-\frac{\delta}{2} < 1.
$$
\end{lemma}

\begin{proof}
Let $A = \begin{bmatrix}
(1-\delta)I   &  0 & -\eta I\\
I & 0 & 0\\
H & -H & 0
\end{bmatrix}$
By definition, if $\lambda$ is an eigenvalue of $A$, then 
$$
\det\left(\begin{bmatrix}
(1-\delta - \lambda)I   &  0 & -\eta I\\
I & -\lambda I & 0\\
H & -H & -\lambda I
\end{bmatrix}\right) = 0.$$ We can then use Section 4.2 of \cite{determinant} to compute the determinant of $A - \lambda I$. Write
$
A - \lambda I = \begin{bmatrix}
S_{11} & S_{12} & S_{13}\\
S_{21} & S_{22} & S_{23} \\
S_{31} & S_{32} & S_{33}
\end{bmatrix}$, we have 
\begin{align*}S_{11} - S_{13}S_{33}^{-1}S_{31} &= (1-\delta - \lambda) I - (-\eta I)(-\frac{1}{\lambda}I)H = (1-\delta-\lambda) I - \frac{\eta}{\lambda} H.\\ S_{12} - S_{13}S_{33}^{-1}S_{32} &= 0 - (-\eta I)(-\frac{1}{\lambda}I)(-H) = \frac{\eta}{\lambda}H. \\
S_{22} - S_{23}S_{33}^{-1}S_{32} &= -\lambda I.\\
S_{21} - S_{23}S_{33}^{-1}S_{31} &= I.
\end{align*}
By Equation 4.8 in \cite{determinant},
\begin{align*}
\det(A - \lambda I) &= \det((1-\delta-\lambda)I - \frac{\eta}{\lambda}H - \frac{\eta}{\lambda}H(-\frac{1}{\lambda}I))\det(-\lambda I)^2\\
&= \det((1-\delta-\lambda)I - \frac{\eta}{\lambda}H + \frac{\eta}{\lambda^2}H)\det(-\lambda I)^2.
\end{align*}
Therefore, if $\lambda$ is an eigenvalue of $A$, it must hold that $\det((1-\delta-\lambda)I - \frac{\eta}{\lambda}H + \frac{\eta}{\lambda^2}H) = 0.$ Let $H = U\Sigma U^\top$ be the eigenvalue decomposition of $H$. Since
$$\det((1-\delta-\lambda)I - \frac{\eta}{\lambda}H + \frac{\eta}{\lambda^2}H) = \prod_{i=1}^d (1-\lambda - \delta - \frac{\eta}{\lambda}\Sigma_{ii} + \frac{\eta}{\lambda^2}\Sigma_{ii}),
$$
it follows that for some $i$,
$1-\lambda - \delta - \frac{\eta}{\lambda}\Sigma_{ii} + \frac{\eta}{\lambda^2}\Sigma_{ii} = 0.
$
Let $\eta\Sigma_{ii} = c_i$, and by our choice of $\eta$, $|c_i|\le \frac{1}{8}$ for all $i\in [d]$. 
We can re-write the above cubic equation 
$$
\lambda^3 - (1-\delta)\lambda^2 + c_i\lambda - c_i = 0,$$
and we will prove the lemma by contradiction. First, observe that $\lambda^3 - (1-\delta)\lambda^2 + c_i\lambda - c_i = 0 \Rightarrow(\lambda^2 + c_i)(\lambda - 1 + \delta) = \delta c_i.$ Suppose $|\lambda|\ge 1-\delta/2\ge3/4.$ By triangle inequality of the complex modulus, $|\lambda-1+\delta| \ge |\lambda| - |1-\delta|\ge \delta/2.$ Since $|\lambda^2 + c_i||\lambda-1+\delta| = \delta|c_i|\ge |\lambda^2+c_i|\delta/2$, it must be that $|c_i| \ge |\lambda^2 + c_i|/2 \ge (|\lambda|^2 - |c_i|)/2$, and $3|c_i|\ge |\lambda|^2\ge 9/16$, which is a contradiction.
\end{proof}

% In meta-optimization, the functions $f_{t, i}$ are allowed to change, so the dynamics of the linear system can be time-varying. We give the analogous definition of strong stability for linear time-varying (LTV) systems below, and assume our system satsifies this condition.

The notion of strong stability cannot be directly applied to LTV dynamical systems, since the dynamics are time-varying. We instead consider sequential stability, a condition on the product of the dynamics matrices. We assume that our system satisfies this condition. 
\begin{definition}[Sequentially stable] A time-varying linear dynamical system is $(\kappa, \gamma)$ sequentially stable if  for all intervals $I = [r, s]\subseteq [T]$,
$
\left \|\prod_{t=s}^{r} A_t\right \| \le \kappa^2 (1-\gamma)^{|I|}.$
\end{definition}

\begin{assumption} \label{assumption:sequential_stable}For settings where the functions $f_{t, i}$ are changing, we assume that the resulting dynamical system is $(\kappa, \gamma)$ sequentially stable with $\kappa \ge 1$. 
\end{assumption}
This assumption is standard in the literature of nonstochastic control for LTV systems \citep{ltv, unknown_ltv}.

\subsection{Optimization algorithms as linear policies.} \label{sec:policy_class}
Given the dynamical system formulation, we next relate optimization algorithms to control policies. To illustrate the nature of optimizers that can be captured by this formulation, we consider stabilizing linear policies for convex quadratics in the deterministic setting, where the objective function $f$ is invariant and the system is LTI. We consider linear policies without loss of generality since the class of DFCs can approximate the class of linear policies as stated in Section \ref{sec:opt_and_control}.

Let $u_t = Kx_t$ be a stabilizing linear policy, then it is a linear function of the current optimization iterate, the past iterate, and the pseudo-gradient. This functional form allows linear policies to 
capture gradient descent, momentum, and preconditioning methods on pseudo-gradients. However, since $K$ is stabilizing, it must satisfy 
$\rho(A+BK) <1,$ and we proceed to characterize the permissible set of values for $K$ using properties of the objective function. We can write
$$
A+BK = \begin{bmatrix}
(1-\delta)I + K_1 & K_2 & -\eta I + K_3\\
I & 0 & 0\\
H & -H & 0
\end{bmatrix},$$
where $\eta, \delta$ are user-defined parameters, and $H$ is the Hessian of $f$.
If $\lambda$ is an eigenvalue of $A+BK$, then
\al{\det\left(\begin{bmatrix}(1-\delta - \lambda)I + K_1 & K_2 & -\eta I + K_3\\
I & -\lambda I & 0\\
H & -H & -\lambda I
\end{bmatrix}\right) = 0.
}
We can compute the determinant by methods developed in \cite{determinant}. Let 
\al{
M_1 =(1-\delta - \lambda)I + K_1 + \frac{(K_3  - \eta I)H}{\lambda},\ \ \ M_2 = K_2 - \frac{(K_3 - \eta I)H}{\lambda},
}
then $
\det(A+BK - \lambda I) = \det(M_1 + \frac{M_2}{\lambda})\det(-\lambda I)^2.$ For $\lambda \neq 0$, this implies that $\det(M_1 + \frac{M_2}{\lambda}) = 0$. Expanding the expression, we have
$$
\det\left((1-\delta - \lambda) I + K_1 + \frac{\lambda-1}{\lambda^2}(K_3 -\eta I)H + \frac{K_2}{\lambda}\right) = 0,$$
suggesting that $\lambda + \delta - 1$ is an eigenvalue of $K_1 + \frac{\lambda-1}{\lambda^2}(K_3 -\eta I)H + \frac{K_2}{\lambda}$. We show in the following subsections that non-trivial algorithms can be expressed as stabilizing linear policies, by upper bounding $|\lambda|$ using this relationship.
\subsubsection{Gradient descent with fixed learning rate.}
We can take $K_1 = K_2 = 0$, $K_3 = \eta 'I$ to encode GD with learning rate $\eta' - \eta$. By the following lemma, any $\eta'$ such that $|\eta' - \eta| \le 1/8\beta$ is a stabilizing linear policy, where $\beta = \|H\|$.
\begin{lemma}
Suppose the conditions in Lemma \ref{lem:stability} are satisfied. Let $K_3 = \eta' I$, then for $\eta'$ such that $|\eta' - \eta| \le 1/8\beta$, for any $\lambda \in \mathbb{C}$ where $\lambda + \delta - 1$ is an eigenvalue of $\frac{\lambda-1}{\lambda^2}(\eta' -\eta)H$, we have $|\lambda| < 1-\delta/2$.
\end{lemma}
\begin{proof}
The proof is similar to the proof of Lemma \ref{lem:stability}. Let $\Sigma_{ii}$ denote the $i$-th eigenvalue of $H$, then $\lambda$ must satisfy, for some $i$,
$
\frac{\lambda-1}{\lambda^2}(\eta' - \eta) \Sigma_{ii} = \lambda  + \delta - 1.$
Rearranging and taking the absolute value, we obtain
$
|\eta-\eta'|\Sigma_{ii}\delta = |\lambda^2 - (\eta'-\eta)\Sigma_{ii}||\lambda -1+\delta|.
$
% First, suppose $\eta = \eta'$, and the left hand side equals zero. Then $|\lambda^2||\lambda-1+\delta| = 0 \Rightarrow \lambda = 0$, or $|\lambda| = 1-\delta < 1-\delta/2$. Assume $\eta \neq \eta'$ and 

Suppose $|\lambda| \ge 1-\delta/2$, then 
\al{
|\eta-\eta'|\Sigma_{ii}\delta &\ge (|\lambda|^2 - |\eta'-\eta|\Sigma_{ii})(|\lambda| - (1-\delta))\ge (|\lambda|^2 - |\eta'-\eta|\Sigma_{ii})\delta/2,
}
and $2|\eta - \eta'|\Sigma_{ii} \ge |\lambda|^2 - |\eta'-\eta|\Sigma_{ii} \Rightarrow  |\eta - \eta'|\Sigma_{ii} \ge (1-\delta/2)^2/3$. Since $|\eta - \eta'| \le 1/8\beta$, $|\eta - \eta'|\Sigma_{ii} \le 1/8$, while the right hand side is at least $3/16$, and we have a contradiction.
\end{proof}

\subsubsection{Momentum.}

% \begin{remark}
% This section gives a formulation to learn the momentum parameter. A more general and efficient formulation is described in section \ref{sec:momentum2}.
% \end{remark}

In this case, $K_3=0$, $K_1 = -vI, K_2 = vI$ describes momentum with parameter $v$, and we show for $v \le \delta$, the corresponding linear policy is stabilizing.
\begin{lemma}
Suppose $\eta, \delta$ satisfy the conditions in Lemma \ref{lem:stability}. Then for $v \in [0, \delta] $, for any $\lambda \in \mathbb{C}$ where $\lambda + \delta - 1$ is an eigenvalue of $-vI - \frac{\eta(\lambda-1)}{\lambda^2}H +\frac{v}{\lambda} I$, we have $|\lambda| < 1-\delta/4$.
\end{lemma}
\begin{proof}
Let $\Sigma_{ii}$ denote the $i$-th eigenvalue of $H$, and let $c_i = \eta\Sigma_{ii}$. Then $\lambda$ must satisfy, for some $i$,
$
\frac{v}{\lambda} - v - \frac{c_i(\lambda-1)}{\lambda^2} = \lambda  + \delta - 1.$
Rearranging and taking the absolute value, we obtain
\al{
|\lambda^2 -v + c_i||\lambda-(1-\delta-v)| = |v-(\delta+v)(v-c_i)|.
}
Since $0 \le v\le 1/2$, the right hand side is equal to $v-(\delta+v)(v-c_i)$. This can be seen as follows: if $v \le c_i$, then the statement is true; if $v \ge c_i$, then $v \ge v-c_i \ge(\delta+v)(v-c_i)$. Assume $|\lambda| \ge 1-\delta/4$, and we show the lemma by contradiction. We first write, 
\al{c_i+(1-\delta-v)(v-c_i) &\ge |\lambda^2 -v + c_i||\lambda-(1-\delta-v)| \\
&\ge (|\lambda|^2 -|v -c_i|)(|\lambda| - (1-\delta-v))\\
&\ge ((1-\delta/4)^2 -|v -c_i|)((1-\delta/4)-(1-\delta-v))
}
If $v\le c_i$, the expression becomes
\al{
c_i + 2(1-\delta-v)(v-c_i) + (c_i-v)(1-\frac{\delta}{4}) &= 2(\delta+v)(c_i-v) +v - \frac{\delta}{4}(c_i-v)\\
&\ge (1-\frac{\delta}{4})^2(\frac{3\delta}{4} +v).
}
Note that the left hand side is upper bounded by $\frac{\delta+v}{4} + v$ because $c_i \le 1/8$, and we have \al{
\frac{\delta}{4} + \frac{5v}{4} \ge \frac{49}{64}(\frac{3\delta}{4} + v),
}
which is a contradiction, because $v\le \delta$. Now, suppose $v> c_i$, we have
\al{
c_i + (v-c_i)(1-\delta/4) + (1-\delta-v)(1-\delta/4)^2\ge (1-\delta/4)^3.
}
We upper bound the left hand side using $1\ge 1-\delta/4$, and obtain $1-\delta \ge(1-\delta/4)^3$, which is a contradiction for $\delta \in [0, 1/2)$.
\end{proof}

\subsubsection{Preconditioned methods.}
Similar to the learning rate case, we set $K_1 = K_2 = 0$, and $K_3 = \eta I -P$, where $P$ is the preconditioner. The following lemma shows that for $P$ such that $\rho(PH) \le 1/8$, the linear policy specified by $K_1, K_2, K_3$ is stabilizing.
\begin{lemma}
Suppose $\eta, \delta$ satisfy the conditions in Lemma \ref{lem:stability}. Then for $P$ such that $\rho(PH) \le 1/8$, for any $\lambda \in \mathbb{C}$ where $\lambda + \delta - 1$ is an eigenvalue of $-\frac{\lambda-1}{\lambda^2}P H$, we have $|\lambda| < 1-\delta/2$.
\end{lemma}
\begin{proof}
Let $c_i$ denote the $i$-th eigenvalue of $PH$. Then for some $i$, 
$
\frac{1-\lambda}{\lambda^2}c_i = \lambda+\delta-1.$ Assume $|\lambda| \ge 1-\delta/2$. After algebraic manipulation and taking the absolute value, we have
\al{|\lambda^2+c_i||\lambda+\delta-1| = \delta|c_i| &\ge (|\lambda|^2-|c_i|)(|\lambda| - (1-\delta)) \ge  (|\lambda|^2-|c_i|)\delta/2.
}
The above inequality implies that $3|c_i| \ge |\lambda|^2 \ge 9/16$, which is a contradiction, since$|c_i|\le 1/8$ by definition.
\end{proof}

For clarity, we only consider optimizers that are either gradient descent, momentum, or preconditioning methods. Combinations of these techniques can also be captured by stabilizing linear policies and can be shown in a similar fashion. Next, we present our main algorithm and its accompanying guarantees.

%% file: alg-n-thm.tex
\section{Algorithm and main theorem statements.}\label{sec:thm+alg}

\begin{algorithm}
\caption{Gradient perturbation controller for meta-optimization}
\begin{algorithmic}[1]
\State Input: $N, T, z_{1, 1}$, $\eta, \delta, \eta_g$, starting points $\{x_{1, i}\}_{i=1}^N$, $\kappa, \gamma$
\State Set: $\M = \{M = \{M^1, \ldots, M^L\}:\|M^l\| \le \kappa^3(1-\gamma)^l\}$, and initialize any $M_{1, 1}\in \M$.
\For{$i = 1, \ldots, N$}
\State If $i>1$, set $z_{1, i} = z_{T+1, i-1}, M_{1, i} = M_{T+1, i-1}$.
\For{$t = 1, \ldots, T$}
\State Choose $u_{t, i}= \sum_{l=1}^L M_{t, i}^l w_{t-l, i}$ and observe $z_{t+1, i}$.
\State Receive $f_{t, i}$, compute $\nabla f_{t, i}(x_{t, i}), \nabla f_{t, i}(x_{t-1, i})$. If $t=T$, compute $w_{T, i}$ by \eqref{eq:reset}. 

\State Suffer control cost $c_{t, i}(z_{t, i}) = f_{t, i}(x_{t, i})$.
\State Construct ideal cost
$g_{t, i}(M) = g_{t, i}(M, \ldots, M)
$ according to (\ref{eq:ideal_cost}).
\State Perform gradient update on the controller parameters, where $\Pi$ denotes projection : 
$$M_{t+1, i} = \Pi_\M (M_{t, i} - \eta_g \nabla g_{t, i}(M_{t, i})).
$$
\EndFor
\EndFor
\end{algorithmic}\label{alg:episodic}
\end{algorithm}
In this section, we give an efficient algorithm that minimizes meta-regret for convex quadratic objective functions and derive its accompanying guarantees. For the rest of the paper, we use the following indexing convention: for the first epoch, let $w_{t, 1} = 0$ for all $t \le 0$; for subsequent epochs, $w_{t, i} = w_{T+t, i-1}, z_{t, i} = z_{T+t, i-1}$ if $t\le0$.  

Our main algorithm, Algorithm \ref{alg:episodic}, views meta-optimization as a single-trajectory control problem, and uses Algorithm \ref{alg:mainA} (GPC) on the dynamical system described by (\ref{eq:dynamics_special}). This algorithm considers the class of DFC controllers, and makes gradient updates on the parameters of interest, $M$, to compete with the best DFC in hindsight. Under Assumption \ref{assumption:sequential_stable}, we do not need a stabilizing linear controller for disturbance-feedback control, and we consider the set of DFCs with $K=0$.

Similar to the GPC method, at each time step, Algorithm \ref{alg:episodic} outputs the control signal generated by $\pi(M_{t, i})$, and receives an objective function. Then, the disturbances are computed according to the dynamical system formulation (\ref{eq:dynamics_special}), and a cost function for updating $M$ is constructed on Line 9 similarly to the surrogate cost in Algorithm \ref{alg:mainA} (GPC). Finally, a gradient update is executed on $M$. We first state the general regret guarantee for Algorithm \ref{alg:episodic} in the next theorem, and then describe the benchmark algorithm class in Section \ref{sec:benchmark}.

\begin{theorem}\label{thm:main}Suppose $f_{t, i}$ are convex quadratic for all $t,\ i$. Under Assumptions \ref{assumption:gradient}, \ref{assumption:reset}, \ref{assumption:sequential_stable},
Algorithm \ref{alg:episodic} with $\eta \le 1, \delta \in (0, \frac{1}{2}]$, $\eta_g = \Theta(\sqrt{NT})^{-1}$, and $L = \Theta(\log NT)$ satisfies
$$ \mregret = \sum_{i=1}^N \sum_{t=1}^T f_{t, i}(x_{t, i}) - \min_{\A\in \Pi} \sum_{i=1}^N \sum_{t=1}^T f_{t, i}(x_{t, i}^{\A}) \le \tilde{O}(\sqrt{NT} ), $$
where $\Theta$ contains polynomial factors in $\gamma^{-1}, \beta, \kappa, R, b, d$, and $\tilde{O}$ contains these polynomial factors as well as logarithmic factors in $T, N$.
% Therefore, on average across epochs, for fixed $\{x_*^i\}_{i=1}^N$, we have:
% \begin{align*} \frac{1}{NT} \sum_{i=1}^N \left [ \sum_{t=1}^T f_{t, i}(x_{t, i}) - \sum_{t=1}^T f_{t, i}(x_*^i)\right] &\leq \frac{1}{NT}\min_{\mA \in \Pi} \sum_{i=1}^N \left [ \sum_{t=1}^T f_{t, i}(x_{t, i}^\mA) - \sum_{t=1}^T f_{t, i}(x_*^i)\right] \\&+ \tilde{O}\left(\frac{1}{\sqrt{TN}}\right).    
% \end{align*}
\end{theorem}

It follows from lower bounds in online convex optimization \cite{hazan2022introduction} that the dependence of meta-regret on $N, T$ are optimal in Theorem \ref{thm:main}. Stronger performance metrics that are  more suitable for changing environments, such as adaptive regret and dynamic regret, were explored in the context of controlling LTV dynamical systems \citep{ltv, unknown_ltv}. Instead of regret against a static comparator, these latter results consider changing comparators, and it is an interesting future direction to extend them to obtain stronger performance guarantees in our setting,  

% \begin{remark}
%     Algorithm \ref{alg:episodic} has access to second-order information of the objective functions, and Newton's method can reach the optimum in one step for quadratics. However, Newton's method requires Hessian inversion, which has practical per-iteration complexity cubic in the dimension. Our method, on the other hand, only needs Hessian-vector products, which can be computed in linear time in the number of nonzero entries, albeit in a higher dimension of the control problem. 
% \end{remark}

The implications of Theorem \ref{thm:main} in the deterministic, stochastic, and adversarial settings are first stated in section \ref{sec:optimization_settings}, and we include them in the next few paragraphs for completeness.

\subsection*{Deterministic optimization}
In this setting, we have an LTI system, and we can obtain more explicit guarantees compared to the stochastic and adversarial settings. First, the sequential stability assumption on the system can be simplified to a standard strong stability (Definition \ref{def:strong_stability}) assumption. Moreover, we can compete against the class of strongly stable linear policies, formally defined below.
\begin{definition}[Strongly stable linear policies] Given a system with dynamics $(A, B)$, a linear policy $K$ is $(\kappa, \gamma)$-strongly stable if $A+BK$ is $(\kappa, \gamma)$ strongly stable, and $\|K\| \le \kappa$.
\end{definition}

As shown in \cite{agarwal2019online}, if our system is $(\kappa, \gamma)$ strongly stable, then the class of DFCs we consider can approximate the class of $(\kappa, \gamma)$ strongly stable linear policies. Lemma \ref{lem:stability} shows that with proper choice of $\eta, \delta$, the system is strongly stable with $\gamma=O(\delta)$. But what does the class of strongly stable linear policies contain? We give some examples of stabilizing linear policies $K$ such that $\rho(A+BK) < 1- O(\delta)$ in Section \ref{sec:policy_class}. By Lemma \ref{lem:tomer}, these policies are also strongly stable with $\gamma = O(\delta)$; therefore, our method can compete with them, and possibly other policies. The form of this class of strongly stable policies $\Pi$ is given in (\ref{eq:lin_policies}), and encompasses certain gradient descent, momentum, and preconditioning methods.

% Theorem \ref{thm:main} implies, 
% $$
% \frac{1}{NT}\sum_{i=1}^N\sum_{t=1}^T f(x_{t, i}) \le \frac{1}{NT} \left[\min_{\A\in \Pi}\sum_{i=1}^N \sum_{t=1}^T f(x_{t, i}^{\A})\right]+ \tilde{O}(\frac{1}{\sqrt{TN}}),$$
% where $\tilde{O}$ hides factors polynomial in $\kappa, \delta^{-1}, \beta, b, d$, and logarithmic in $T, N$.
Let $\bar{x} = \frac{1}{TN}\sum_{i=1}^N\sum_{t=1}^T x_{t, i} $ be the average iterate, and $\bar{J}(\A) = \frac{1}{TN} \sum_{i=1}^N J_i(\A)$ denote the average cost of the algorithm $\A$. Then by convexity, Theorem \ref{thm:main} implies
$$ f(\bar{x}) \leq \min_{\A\in \Pi} \bar{J}(\A) + \tilde{O}\left (\frac{1}{\sqrt{TN}}\right ),$$
where $\Pi$ is the class of strongly stable linear policies.
% That is, the average per-epoch function value approaches the per-epoch value of the best optimization algorithm from the family $\Pi$.
\subsection*{Stochastic optimization}
In this setting, our functions are drawn randomly from distributions $\D_1, \D_2, \ldots, \D_N$ that vary from epoch to epoch. In epoch $i$, for each time step $t\in [T]$, we draw a quadratic function $f_{t, i} \sim \D_i$. 
Let $\E$ denote the unconditional expectation with respect to the randomness of the functions, and define the function $\bar{f}_i(x) := \E_{\D_i}[f_{t, i}(x)]$, then the guarantee can be written as

\begin{align*}
\frac{1}{NT}\sum_{i=1}^N\sum_{t=1}^T \E[\bar{f}_i(x_{t, i})]
&\le \frac{1}{NT} \min_{\A\in \Pi}\sum_{i=1}^N \sum_{t=1}^T \E\left[\bar{f}_{i}(x_{t, i}^\A)\right]+ \tilde{O}\left (\frac{1}{\sqrt{TN}}\right),
\end{align*}
where $\Pi$ is the benchmark algorithm class described in the next subsection.

\subsection*{Adversarial online optimization}
Consider the setting where we have a new function at each $t, i$ in the optimization process.
% These functions can arrive in an online and adversarial manner; in other words, they do not satisfy distributional assumptions such as ones presented above. 
This setting describes the meta-online convex optimization (meta-OCO) problem, and we give our guarantees in the standard OCO metric -- regret. 

Let $x^*_{i} \in \argmin_x \sum_{t=1}^T f_{t, i}(x)$ be an optimum in hindsight in episode $i$, and denote $\regret_i(\A)$ as the regret suffered by the algorithm $\A$ in epoch $i$. Subtracting $\sum_{i=1}^N \sum_{t=1}^T f_{t, i}(x_i^*)$ on both sides,
\begin{align*}
\frac{1}{TN}\sum_{i=1}^N \text{Regret}_i &\le \min_{\A\in \Pi}\frac{1}{TN}\sum_{i=1}^N\text{Regret}_i(\A) + \tilde{O}\left (\frac{1}{\sqrt{NT}}\right).
\end{align*}
%That is, over episodes, the average regret approaches that of the best online learner in the class $\Pi$, described below.

\subsection{The benchmark algorithm class.}\label{sec:benchmark}
We introduce the benchmark algorithm class for LTV systems, which are present in the stochastic and adversarial online meta-optimization settings. This benchmark algorithm class consist of optimizers that correspond to a class of DFCs and  generalizes the strongly stable linear policies discussed previously. Informally, our guarantee is competitive with optimizers that are linear functions of past gradients. 

We start with an expression of the system state under a DFC controller. Let $x_{t, i}^M$ denote the optimization iterate at time $(t, i)$ under controller $\pi(M)\in \Pi_{DFC}$. Define the \textit{pseudo-gradient} as
$$
\hat{\nabla}f_{t, i}^M(\cdot) = \nabla f_{t, i}(\cdot) - \nabla f_{t, i}(x_{t-1, i}^M) + \nabla f_{t, i}(x_{t-1, i}),$$
where $x_{t-1, i}$ is the optimization iterate chosen by our algorithm at time $(t-1, i)$. It is clear that the pseudo-gradient is close to the true gradient $\nabla f_{t, i}(\cdot)$ if $x_{t-1, i}$ and $x_{t-1, i}^M$ are close. The state under $\pi(M)$ is
$$
z_{1, i}^M = \begin{bmatrix}x_{1, i} \\x_{1, i}\\0
\end{bmatrix},\ \ z_{2, i}^M = \begin{bmatrix}x_{2, i}^M \\x_{1, i}\\\nabla f_{1, i}(x_{1, i})
\end{bmatrix},\ \ z_{t, i}^M = \begin{bmatrix}x_{t, i}^M \\x_{t-1, i}^M\\\hat{\nabla}f_{t-1, i}^M(x_{t-1, i}^M)
\end{bmatrix}\text{  for  } t\ge 3.$$

Observe that the state contains the pseudo-gradient at location $x_{t-1, i}^M$, instead of the true gradient under controller $M$, $\nabla f_{t-1, i}(x_{t-1, i}^M)$, because our guarantee applies to a fixed sequence of disturbances, cost functions, and dynamics. The dynamical system formulation (\ref{eq:dynamics_special}) describes the exact evolution of the gradients $\nabla f_{t, i}(x_{t, i})$, along the trajectory chosen by our algorithm; however, under the trajectory induced by another controller $M$ and fixed disturbances, the system instead describes the evolution of the pseudo-gradients.   

The state evolution under $\pi(M)$ corresponds to the update of the following optimizer:
\begin{align}\label{eq:benchmark_algo_class}
    x_{t+1, i}^M = (1-\delta) x_{t, i}^M - \eta \hat{\nabla}f_{t-1, i}^M(x_{t-1, i}^M) + \sum_{l=1}^L M^l\nabla f_{t-l, i}(x_{t-l-1, i}).
\end{align}

Therefore, competing with the class of DFCs translates to competing with the class of optimizers parameterized by $M \in \mathcal{M}$, with updates specified by (\ref{eq:benchmark_algo_class}). The set $\mathcal{M}$ is defined in Algorithm \ref{alg:episodic}. As discussed above, this benchmark algorithm class has a more straightforward interpretation for deterministic meta-optimization, capturing common algorithms on pseudo-gradients, including gradient descent, momentum, and preconditioned methods. 

\subsection{Example: competing with the best learning rate for convex quadratics.} \label{sec:example}
For illustration, we include an example of meta-optimization where we compete with the best gradient descent learning rate (from a set) for convex quadratics. Consider the deterministic setting , where we receive a quadratic objective function $f(x) = \frac{1}{2}x^\top H x$ . Assume $\|H\| = \beta \ge 1$. For gradient descent, a good choice of learning rate is $\frac{1}{\beta}$, but often we only have an upper bound $\hat{\beta}$ such that $\hat{\beta} \gg \beta$. As we show in the sequel, we can do almost as well as gradient descent, with learning rate $\frac{1}{8\beta}$ and pseudo-gradients, on average using meta-optimization.

Suppose we choose $\eta_g$, $L, \delta$ according to Theorem \ref{thm:main}, and set $\eta = \frac{1}{8\hat{\beta}}$. Then by Lemma \ref{lem:stability}, the dynamical system is stable and satisfies Definition \ref{def:strong_stability}. Moreover, Assumption \ref{assumption:gradient} is satisfied with $\hat{\beta}$, and if Assumption \ref{assumption:reset} holds, by Theorem \ref{thm:main} we can compete with the best stabilizing linear policy. For a linear policy $K$, let 
$$
z_{t, i}^K = \begin{bmatrix}
    x_{t, i}^K\\ x_{t-1, i}^K\\ \hat{\nabla} f^K(x_{t-1, i}^K)
\end{bmatrix}$$ denote the state reached at time $(t, i)$ by playing policy $K$. Let $[K_1\ K_2\  K_3] \in \reals^{d\times 3d}$ represent the top $d$ rows of $K$, where the submatrices have dimension $d\times d$. The closed-loop dynamics of the linear policy $K$ is 
\begin{align*}
    z_{t+1, i}^K = \begin{bmatrix}
(1-\delta)I+K_1 & K_2   &  -\frac{1}{8\hat{\beta}} I+K_3\\
I & 0 & 0 \\
 H & -H & 0
\end{bmatrix} z_{t, i}^K + w_{t, i}.
\end{align*}
Setting $K_1 = 0$, $K_2 = 0$, $K_3 = -(\frac{1}{8\beta} - \frac{1}{8\hat{\beta}})I$, the dynamics is gradient descent (using pseudo-gradients) with learning rate $\frac{1}{8\beta}$ and weight decay. By Lemma \ref{lem:stability}, this closed-loop dynamics is stable, so our choice of $K$ is a stabilizing linear policy, and we do at least as well as playing $K$ on average.

%% file: smooth_convex.tex
\section{Smooth convex meta-optimization.} \label{sec:smooth}
In this section, we present guarantees for meta-optimization when the objective functions are smooth and convex, a significantly broader class of functions than convex quadratics. However, we need to modify the dynamical system formulation slightly due to the absence of quadratic structure in these functions. For the sequel, we assume the following two assumptions are satisfied:  Assumption \ref{assumption:bounded_hessian} (smoothness) and  Assumption \ref{assumption:bounded} (boundedness), formaly defined as follows. 

\begin{assumption} \label{assumption:bounded_hessian} The objective functions $f_{t, i}$ have uniformly bounded Hessians,
$
\|\nabla^2 f_{t, i}(x)\| \le \beta,\ \ \forall\ x, \ t,\ i. $
\end{assumption}
The assumption above implies that he objective functions have Lipschitz gradients, 
\begin{equation} \label{eq:smooth_grad}
\|\nabla f_{t, i}(x) - \nabla f_{t, i}(y)\| \le \beta \|x-y\|,\text{ and }
    \|\nabla f_{t, i}(x) \| \le \beta \|x\| + b
\end{equation}for some $b\ge 0$, for all $x, t, i$.

\begin{assumption}\label{assumption:bounded}
    The objective functions are bounded: $f_{t, i}(x) \le C$ for all $t, i, x$.
\end{assumption}
\subsection{The dynamics of smooth convex meta-optimization.}
The main difference in the dynamical system formulation between quadratic and smooth meta-optimization is the evolution of the gradients. We would like to find some matrix $H_{t, i}$ with bounded spectral radius that satisfies
\begin{align} \label{eq:gradient_evolution}
    \nabla f_{t, i}(x_{t, i}) = \nabla f_{t, i}(x_{t-1, i}) + H_{t, i}(x_{t, i} - x_{t-1, i}).
\end{align} Since the mean value theorem has no multi-variate analogue, we cannot guarantee in general that $H_{t, i}$ is the Hessian of $f_{t, i}$ at some location between $x_{t, i}$ and $x_{t-1, i}$. However, we can consider $H_{t, i}$ that contains coordinate-wise second order gradients at different locations. More precisely,
consider the gradient of $f_{t, i}$, $\nabla f_{t, i} \in \reals^{d}$, and let $\nabla f_{t, i}^j (x)$ denote the $j$-th coordinate of the gradient. Let $\nabla^2f_{t, i}^j(x) = \frac{\partial \nabla f_{t, i}^j(x)}{\partial x} \in \reals^{d}$ be the gradient of $\nabla f_{t, i}^j(x)$, and define 
$$H_{t, i}(y_1, \ldots, y_d) = \begin{bmatrix}
    \nabla^2f_{t, i}^1(y_1)\\
    \vdots\\
    \nabla^2f_{t, i}^d(y_d)
\end{bmatrix}.
$$
Namely, $H_{t, i}(y_1, \ldots, y_d)$ is a matrix with rows that are gradients of $\nabla f_{t, i}^1, \ldots, \nabla f_{t, i}^d $ at locations $y_1, \ldots, y_d$.
The system evolution is
\begin{equation}\label{eq:dynamics_special_smooth}
\begin{bmatrix}
x_{t+1, i} \\
x_{t, i}\\
\nabla f_{t, i}(x_{t, i})
\end{bmatrix} 
= 
\begin{bmatrix}
(1-\delta)I & 0   &  -\eta I\\
I & 0 & 0 \\
 H_{t, i} & -H_{t, i} & 0
\end{bmatrix} 
\times
\begin{bmatrix}
x_{t, i} \\
x_{t-1, i}\\
\nabla f_{t-1, i}(x_{t-1, i})
\end{bmatrix} 
+ 
\begin{bmatrix}
I & 0 & 0 \\
0 & 0 & 0\\
0 & 0 & 0
\end{bmatrix}  
\times
u_{t, i} 
 + \begin{bmatrix}
 0\\
 0\\
 \nabla f_{t, i}(x_{t-1, i})
 \end{bmatrix},
 \end{equation}
where $H_{t, i}$ satisfies
$
H_{t, i} = H_{t, i}(\xi^1_{t, i}, \ldots, \xi^d_{t, i})$
for some $\xi_{t, i}^j$ on the line segment from $x_{t-1, i}$ to $x_{t, i}$, for all $j\in [d]$. 
Applying the mean value theorem to real-valued functions $\nabla f_{t, i}^j(x)$, we can find $\{\xi_{i, t}^j\}_{j=1}^d$ such that (\ref{eq:gradient_evolution}) is satisfied. 
Further, we make the following assumption on $H_{t, i}$ uniformly, so that Lemma \ref{lem:stability} holds for this formulation as well. 

\begin{assumption}\label{assumption:smooth}
    For all $(t, i)$, $\rho(H_{t, i}) \le \beta.$
\end{assumption}
Note that the proof of Lemma \ref{lem:stability} requires $H_{t, i}$ to be symmetric, a condition not satisfied by this formulation. However, instead of using the singular value decomposition in the proof, we can use the Schur triangularization of a square matrix. As long as Assumption \ref{assumption:smooth} is satisfied, the system is stable for $\eta \le \frac{1}{8\beta}$ and $\delta\in (0, \frac{1}{2}]$.

In the case of smooth quadratic objective functions, $H_{t, i}$ is the Hessian, and Assumption \ref{assumption:smooth} is subsumed by Assumption \ref{assumption:bounded_hessian}. It can also be satisfied by smooth convex functions whose Hessians have row norms bounded by $\beta/\sqrt{d}$ uniformly, since for such functions, $\rho(H_{t, i}) \le \|H_{t, i}\| \le \|H_{t, i}\|_F \le \beta$.

Note that $H_{t, i}$ is in fact not directly observable to us. {\bf  Crucially, however, we do not need $H_{t, i}$ or any system information for the algorithm we develop}; we only need to know the disturbances, which can be computed by taking gradients of the objective functions.

Other components of the dynamical system formulation, including the system resets and stability assumptions on the system, remain the same for convex smooth meta-optimization.

% \paragraph{Stability.} We consider problems where Assumption \ref{assumption:sequential_stable} holds on the linearized dynamics. 

\subsection{Algorithm and guarantees.}
Convex smooth meta-optimization can be treated as a single-trajectory control problem, similar to quadratic meta-optimization. However, there is an additional challenge that the dynamics are unknown. In standard nonstochastic control settings, not knowing the dynamics means that both the disturbances and the gradients with respect to the controller parameters $M$ need to be estimated. Since we can directly compute the disturbances in (\ref{eq:dynamics_special_smooth}) by taking gradients of the objective functions, we only need to modify Algorithm \ref{alg:episodic} to incorporate gradient estimation.

In this setting, we have access to the control costs, so it is natural to consider bandit algorithms that estimate the gradients with access to only the cost function values. Let $\M_{\delta_M} = \{M \in \M: \frac{1}{1-\delta_M} M \in \M\}$. The algorithm for convex smooth meta-optimization is given in Algorithm \ref{alg:meta_opt_smooth}, where we use a bandit variant of the GPC method developed in \cite{gradu2020non} (Algorithm 2), and briefly mention their techniques in the Appendix. Alternatively, bandit algorithms from \cite{ghai2023online} can be applied to yield potentially better dimension dependence.

Algorithm \ref{alg:meta_opt_smooth} estimates the gradients with respect to $M$ by perturbing the inputs and observing the cost values on the perturbed inputs, similar to the FKM method in \cite{flaxman2005online}. In each iteration, we sample stochastic noise from the unit sphere and produce the perturbed $\widetilde{M}_{t, i}$ on line 13. Then we play controls with parameters $\widetilde{M}_{t, i}$, record the cost, and construct the gradient estimate $g_{t, i}$ on line 11. We compute the disturbances on line 9 to execute the DFC policy.

\begin{algorithm}\label{alg:meta_opt_smooth}
\caption{Smooth convex meta-optimization}
\begin{algorithmic}[1]
\State Input: $N, T, z_{1, 1}$, $\eta, \delta, \{\eta_{t, i}^g\}, L, \delta_M$, starting points $\{x_{1, i}\}_{i=1}^N$, $\kappa, \gamma$
\State Set: $\M = \{M = \{M^1, \ldots, M^L\}:\|M^l\| \le \kappa^3(1-\gamma)^l\}$.
\State Initialize any $M_{1, 1} = \cdots = M_{L, 1} \in \M_{\delta_M}$.
\State Sample $\epsilon_{1, 1}, \ldots, \epsilon_{L, 1} \in_{\mathbb{R}} \mathbb{S}_1^{L \times 3n \times 3n}$, set $\widetilde{M}_{l, 1} = M_{l, 1} + \delta_M \epsilon_{l, 1}$ for $l = 1, \ldots, L$.
\For{$i = 1, \ldots, N$}
\State If $i>1$, set $z_{1, i} = z_{T+1, i-1}, M_{1, i} = M_{T+1, i-1}$.
\For{$t = 1, \ldots, T$}
\State Choose $u_{t, i}= \sum_{l=1}^L \widetilde{M}_{t, i}^l w_{t-l, i}$.
\State Receive $f_{t, i}$, compute $\nabla f_{t, i}(x_{t-1, i})$. If $t=T$, then compute $w_{T, i}$ by \ref{eq:reset}.
\State Suffer control cost $c_{t, i}(z_{t, i}) = f_{t, i}(x_{t, i})$.
\State Store $g_{t,i} = \dfrac{9n^2L}{\delta_M} c_{t, i}(z_{t, i}) \sum\limits_{l=1}^{L} \epsilon_{t-l, i} \:$ if $t \geq L$ else 0.
\State Perform gradient update on the controller parameters: 
$$M_{t+1, i} = \Pi_{\M_{\delta_M}} (M_{t, i} - \eta_{t, i}^g g_{t-L, i}).
$$
\State Sample $\epsilon_{t+1, i} \in_\text{R} \mathbb{S}_{1}^{L \times 3n \times 3n}$,  set $\widetilde{M}_{t+1, i} = M_{t+1, i} + \delta_M \epsilon_{t+1, i}$
\EndFor
\EndFor
\end{algorithmic}
\end{algorithm}

The regret guarantee for Algorithm \ref{alg:meta_opt_smooth} is given in the theorem below. We show that the meta-regret is sublinear in expectation, with a rate of $(NT)^{3/4}$. The dependence of meta-regret on $NT$, the learning horizon, is due to the regret of the FKM method in the classical bandit convex optimization (BCO) setting. The recent work of \citet{sun2023optimal} improves bandit GPC to have regret $T^{1/2}$ under quadratic strongly convex losses.
\begin{theorem}\label{thm:smooth}Under Assumptions \ref{assumption:reset}, \ref{assumption:sequential_stable}, \ref{assumption:bounded_hessian}, 
\ref{assumption:bounded}, \ref{assumption:smooth}, 
Algorithm \ref{alg:meta_opt_smooth} with $\eta \le 1$, $L = \Theta(\log NT)$, and setting $\eta_{t, i}^g = \Theta((N(i-1)+t)^{-3/4} L^{-3/2})$ and perturbation constant $\delta_M = \Theta((NT)^{-1/4} L^{-1/2})$ satisfies
$$ \E\left[\mregret\right] = \E\left[\sum_{i=1}^N \sum_{t=1}^T f_{t, i}(x_{t, i})\right] - \min_{\A\in \Pi} \sum_{i=1}^N \sum_{t=1}^T f_{t, i}(x_{t, i}^{\A}) \le \tilde{O}((NT)^{3/4} ), $$
where $\tilde{O}$, $\Theta$ contain polynomial factors in $\gamma^{-1}, \beta, \kappa, R, b, d, C$, and $\tilde{O}$ in addition contains logarithmic factors in $T, N$. The benchmark algorithm class $\Pi$ is the class of DFCs.
\end{theorem}

Similar to convex quadratic meta-optimization, Theorem \ref{thm:smooth} can be refined for the deterministic, stochastic, and adversarial smooth meta-optimization settings. We give a proof of Theorem \ref{thm:smooth} in the next section, for details and background on nonstochastic control techniques in the bandit setting, see Appendix \ref{sec:bco}.

\subsubsection*{The benchmark algorithm class}
    Since convex smooth functions have location-dependent Hessians, we give a different definition for pseudo-gradients from the one in Section \ref{sec:benchmark}. We define them as follows,
    $$\hat{\nabla}f_{t, i}^M(x) = H_{t, i}(x - x_{t-1, i}^M) + \nabla f_{t, i}(x_{t-1, i}).
    $$
    The corresponding benchmark algorithm class has the same form with the modified definition of pseudo-gradients,
    \begin{align*}
    x_{t+1, i}^M = (1-\delta) x_{t, i}^M - \eta \hat{\nabla}f_{t-1, i}^M(x_{t-1, i}^M) + \sum_{l=1}^L M^lw_{t-l, i}.
\end{align*}

%% file: proofs.tex
\section{Proof of main theorems.}\label{app:main_thm_proof}
We give the proofs of Theorem \ref{thm:main} and Theorem \ref{thm:smooth} in this section. We consider Algorithms \ref{alg:episodic} and \ref{alg:meta_opt_smooth} as gradient perturbation controllers on the dynamical system given by (\ref{eq:dynamics_special}) and (\ref{eq:dynamics_special_smooth}), respectively, and bound the meta-regret by the regret of the control algorithms against class of DFCs. 

However, the analysis of the GPC method in nonstochastic control does not apply directly to meta-optimization, and new techniques are required to derive the regret guarantees. In particular, all existing methods in nonstochastic control rely on the assumption of a uniform upper bound on $\|w_{t, i}\|$, which is not granted in our case since the disturbances include $\nabla f_{t, i}(x_{t-1, i})$. Under the smoothness assumption, the disturbances $\|w_{t, i}\|$ can grow proportionally with the states $\|x_{t-1, i}\|$. Furthermore, the reset disturbance scales with the size of the states themselves, and this interdependence between the states and disturbances presents a further challenge in the analysis.

We overcome this challenge by showing that due to system stability, the effect of the reset disturbance attenuates exponentially. In each episode, after the initial large reset disturbance, the state and disturbance have an upper bound that decreases with time and eventually converges to a constant value. Moreover, we can scale the functions and bound the contribution from $w_{t, i}$ to the size of the state. To make this intuition formal, we use several inductive arguments to address the interdependence between the states and disturbances. 

\subsection{Proof of Theorem \ref{thm:main}.}
The regret of Algorithm \ref{alg:episodic} is the excess total cost of the controller compared to that of the best fixed DFC in hindsight. Since the system is strongly stable, we can approximate the instantaneous cost of any controller, which is a function of all the previous states and controls, with a surrogate cost that is a function of only the past $L$ control signals, where $L$ is logarithmic in the horizon $NT$. Bounding the regret then reduces to obtaining low regret on the surrogate costs, which was studied in the online learning with memory (OCOwM) framework \cite{anava2015online}. 

Similar to the proof of the main theorem in \cite{agarwal2019online}, we decompose the regret into three terms: the errors due to approximating the costs of our controller and the best DFC in hindsight with the surrogate cost functions, and online learning of the best DFC given the surrogate costs. We bound these terms separately in Sections \ref{sec:approx_error} and \ref{sec:ocowm_regret}. Crucially, these results rely on universal upper bounds on the states and disturbances, which is the main technical challenge of this work and given in Section \ref{sec:state_bound}.       

\begin{proof} Observe that the control-input matrix $B_t$ in formulation \ref{eq:dynamics_special} is time-invariant, and $\|B_t\| = 1$. Let $A_{s:t, i} = \prod_{k=s}^t A_{k, i}$ denote the product of the dynamics matrices from time $s$ to $t$ in epoch $i$. Define the surrogate state $y_{t, i}$ as the state reached at time $(t, i)$, if $z_{t-L, i} = 0$ and we play the sequence of policies $M_{t-L, i}, \ldots, M_{t-1, i}$. More precisely, let $\mathbbm{1}$ be the indicator function, by expanding the linear dynamics, the surrogate state $y_{t, i}(M_{t-L, i}, \ldots, M_{t-1, i})$ can be written as
\begin{align*}
\sum_{k=1}^{2L}\left(\sum_{j=1}^{L}A_{t-j+1:t-1}BM^{k-j}_{t-j, i}\mathbbm{1}_{k-j\in[1, L]} + A_{t-k+1:t-1} \mathbbm{1}_{k\le L}\right)w_{t-k, i}.
\end{align*}
In the expression above, we take the coefficient of $w_{t-1, i}$ to be $I$, and $A_{t+1:t-1} = I$.
We omit the policies and write $y_{t, i}$ whenever the executed policy is clear. 
Define the surrogate cost as the cost at the surrogate state,
\begin{equation}\label{eq:ideal_cost}
g_{t, i}(M_{t-L, i}, \ldots, M_{t-1, i}) = c_{t, i}(y_{t, i}(M_{t-L, i}, \ldots, M_{t-1, i})).\end{equation}
Let $z_{t, i}(M)$ denote the state reached by playing the policy $M$ across all time steps. Let $\M$ be defined as in Algorithm \ref{alg:episodic}. The regret decomposition can be written as,
\begin{align}
&\sum_{i=1}^N \sum_{t=1}^T c_{t, i}(z_{t, i}) - \min_{M\in \M}\sum_{i=1}^N\sum_{t=1}^T c_{t, i}(z_{t, i}(M)) \\
&\le \sum_{i=1}^N \sum_{t=1}^T c_{t, i}(z_{t, i}) - \sum_{i=1}^N \sum_{t=1}^T g_{t, i}(M_{t-L, i}, \ldots, M_{t-1, i}) \label{eq:2}\\&+ \sum_{i=1}^N \sum_{t=1}^T g_{t, i}(M_{t-L, i}, \ldots, M_{t-1, i}) - \min_{M\in \M}\sum_{i=1}^N \sum_{t=1}^T g_{t, i}(M, \ldots, M)\\
&+ \min_{M\in \M}\sum_{i=1}^N \sum_{t=1}^T g_{t, i}(M, \ldots, M) - \min_{M\in \M}\sum_{i=1}^N \sum_{t=1}^T c_{t, i}(z_{t, i}(M)) \label{eq:3}.
\end{align}
The first and third terms are approximation error terms due to using the surrogate states and costs, and the second term bounds the regret of OCOwM.
Note that if we scale the cost functions by a constant $c$, then the regret also scales with $c$. Define $$\beta_0 = \frac{\gamma(1-\gamma)^2}{10^3\kappa^5 L^2},$$ and assume $f_{t, i}$ satisfies Assumption \ref{assumption:gradient} with $\beta \le \beta_0$. If the assumption is not satisfied, we can scale $f_{t, i}$ by $\beta_0/\beta$, and multiply the obtained meta-regret bound by $\beta/\beta_0$. Define the state magnitude upper bound
$$D = (80\kappa^2La+1)(2\kappa^2 R  + \frac{32\kappa^2 Lab}{\gamma})+ 8\kappa^2Lab,
$$ 
where we take $a= \kappa^3$. With this choice of $a$, it can be shown that if the system is LTI, then the resulting class of disturbance action policies can approximate strongly stable linear policies \cite{hazan2022introduction}.
By Lemma \ref{lem:truncation_bound}, we can bound the approximation error terms as follows,
$$ (\ref{eq:2}) + (\ref{eq:3}) \le 2TND^2\kappa^2(1-\gamma)^{L} + \frac{2\kappa^2D^2}{\gamma}.
$$
By Lemmas \ref{lem:f_Lipschitz}, \ref{lem:f_Gradient}, and Theorem 3.1 in \cite{anava2015online}, the regret for OCOwM satisfies 
\begin{align*}
&\sum_{i=1}^N \sum_{t=1}^T g_{t, i}(M_{t-L, i}, \ldots, M_{t-1, i}) - \min_{M\in \M}\sum_{i=1}^N \sum_{t=1}^T g_{t, i}(M, \ldots, M) \le O\left(\frac{\kappa^5d^{3/2}L^2D^2}{\gamma^2} \sqrt{NT}\right).
\end{align*}
The result follows by summing up the three terms and setting $L = O(\frac{\kappa^2}{\gamma}\log TN)$.
\end{proof}

\subsubsection{Bounding the approximation errors.}\label{sec:approx_error}
The following two lemmas establish an upper bound on the approximation error terms. Lemma \ref{lem:truncation_bound} obtains the result by analyzing the per-iteration error using a gradient upper bound on the control costs, and the exponentially decaying distance between the actual and the surrogate states, shown by Lemma \ref{lem:ideal_states}.
\begin{lemma}\label{lem:truncation_bound}
 Assume the conditions of Theorem \ref{thm:state_bound} are satisfied, and define 
 $$D = (80\kappa^2La+1)(2\kappa^2 R  + \frac{32\kappa^2 Lab}{\gamma})+ 8\kappa^2Lab,
$$ then the approximation error due to using the surrogate states and costs satisfies
\begin{align*}
\sum_{i=1}^N\sum_{t=1}^Tg_{t, i}(M_{t-L, i}, \ldots, M_{t-1, i}) - \sum_{i=1}^N\sum_{t=1}^Tc_{t, i}(z_{t, i}) \le TND^2\kappa^2(1-\gamma)^{L} + \frac{\kappa^2D^2}{\gamma}.
\end{align*}
The same result holds for 
$$\min_{M\in \M}\sum_{i=1}^N \sum_{t=1}^T g_{t, i}(M, \ldots, M) - \min_{M\in \M}\sum_{i=1}^N \sum_{t=1}^T c_{t, i}(z_{t, i}(M)).$$ 
\end{lemma}
\begin{proof}
Using the mean value theorem, 
\begin{align*}
|c_{t, i}(z_{t, i})  - g_{t, i}(M_{t-L, i}, \ldots, M_{t-1, i})| &= |c_{t, i}(z_{t, i}) - c_{t, i}(y_{t, i}(M_{t-L, i}, \ldots, M_{t-1, i}))|\\
&\le \|\nabla c_{t, i}(\xi_{t,i})\|\|z_{t, i} - y_{t, i}(M_{t-L, i}, \ldots, M_{t-1, i})\|,
\end{align*}
where $\xi_{t, i} = \lambda z_{t, i} + (1-\lambda)y_{t, i}$ for some $\lambda\in [0, 1]$. Since $\|\nabla c_{t, i}(z_{t, i})\| = \|\nabla f_{t, i}(x_{t, i})\| \le 2\beta\|x_{t, i}\|+b\le 2\beta\|z_{t, i}\|+b$, we have $\|\nabla c_{t, i}(\xi_{t, i})\| \le 2\beta \max\{\|z_{t, i}\|, \|y_{t, i}\|\} + b$. By Theorem \ref{thm:state_bound} , the states are bounded by $D$, and we can conclude the same for the surrogate states by using a similar argument. Moreover, we have 
$
2\beta D+b \le D,
$ and by Lemma \ref{lem:ideal_states}, the difference in one time step satisfies
\begin{align*}
|c_{t, i}(z_{t, i})  - g_{t, i}(M_{t-L, i}, \ldots, M_{t-1, i})| &\le (2\beta D+b)\kappa^2(1-\gamma)^{L}D\le \kappa^2(1-\gamma)^{L}D^2,
\end{align*}
for $i > 1$ or $t\ge L+1$, and 
\begin{align*}
|c_{t, i}(z_{t, i})  - g_{t, i}(M_{t-L, i}, \ldots, M_{t-1, i})| &\le \kappa^2D^2(1-\gamma)^{t-1},
\end{align*}
for $i=1$, $t\le L$. Summing over episodes gives the first inequality. For the second expression, let $M^* \in \argmin \sum_{i=1}^N \sum_{t=1}^T c_{t, i}(z_{t, i}(M)) $, we have
\begin{align*}
    &\min_{M\in \M}\sum_{i=1}^N \sum_{t=1}^T g_{t, i}(M, \ldots, M) - \min_{M\in \M}\sum_{i=1}^N \sum_{t=1}^T c_{t, i}(z_{t, i}(M)) \\ &\le \sum_{i=1}^N \sum_{t=1}^T g_{t, i}(M^*, \ldots, M^*) - \sum_{i=1}^N \sum_{t=1}^T c_{t, i}(z_{t, i}(M^*)), 
\end{align*}
and we can obtain the same result.
\end{proof}

\begin{lemma}\label{lem:ideal_states}
Define 
$D 
$ as in Lemma \ref{lem:truncation_bound},
then under the conditions of Theorem \ref{thm:state_bound}, for all $t, i$ the surrogate states are bounded by $D$. Furthermore, for $i >1$ and $t\ge 1$, or $i=1$ and $t \ge L+1$,
$$
\|z_{t, i} - y_{t, i}(M_{t-L, i}, \ldots, M_{t-1, i})\| \le \kappa^2(1-\gamma)^{L}D,$$ and for $t \le L+1$, 
$$
\|z_{t, 1} - y_{t, 1}(M_{t-L, 1}, \ldots, M_{t-1, 1})\| \le \kappa^2(1-\gamma)^{t-1}D.
$$
\end{lemma}
\begin{proof}
By the definition of the surrogate state, 
$$
\|z_{t, i} - y_{t, i}\|\le \|A_{t-L:t-1, i}\|\|z_{t-L, i}\| \le \kappa^2(1-\gamma)^{L}D$$
 for $i\ge 2$ or $t\ge L+1$ and $i=1$. For $t\le L$, $i=1$,
$$
\|z_{t, i} - y_{t, i}\|\le \|A_{t-L:t-1, i}\|\|z_{t-L, i}\| \le \kappa^2(1-\gamma)^{t-1}D.$$
\end{proof}

\subsubsection{Bounding the OCOwM regret.}\label{sec:ocowm_regret}
The OCOwM framework was studied in \cite{anava2015online}, and we briefly summarize the result for completeness. Consider the following online learning task with memory: at time $t$, the player chooses $x_t\in \mathcal{K} \subset\reals^d$, receives the loss function $g_t:\mathcal{K}^{L} \rightarrow \reals$, and suffers the loss $g_t(x_{t-L+1}, \ldots, x_t)$. Suppose the $g_t$'s have a coordinate-wise Lipschitz property and satisfies
$$
|g_t(x_1, \ldots, x_j, \ldots, x_L) - g_t(x_1, \ldots, \tilde{x}_j, \ldots, x_L)| \le L_g\|x_j - \tilde{x}_j\|.
$$
Let $\tilde{g}_t(x) = g_t(x, \ldots, x)$ be a unary function, and define the gradient upper bound and the diameter of $\mathcal{K}$ as
$$
G_g = \sup_{t\in \{1, \ldots, T\}, x\in \mathcal{K}}\|\nabla \tilde{g}_t(x)\|, \ \ D_{\mathcal{K}} = \sup_{x, y\in \mathcal{K}}\|x - y\|.$$
Then running OGD on the unary losses $\tilde{g}_t$ with learning rate $\eta = \frac{D_{\mathcal{K}}}{\sqrt{G_g(G_g + L_gL^2)T}}$ outputs a sequences of decisions such that
$$
\sum_{t=L}^Tg_t(x_{t-L+1}, \ldots, x_t) - \min_{x\in \mathcal{K}}\sum_{t=L}^Tg_t(x, \ldots, x) \le O\left(D_{\mathcal{K}}\sqrt{G_g(G_g+L_gL^2)T}\right).$$
We proceed to bound the Lipshitz constant $L_g$ of our surrogate losses $g_{t, i}$ in Lemma \ref{lem:f_Lipschitz}, and the gradient upper bound $G_g$ in Lemma \ref{lem:f_Lipschitz}. 
\begin{lemma}\label{lem:f_Lipschitz}
    Define $D$ as in Lemma \ref{lem:truncation_bound}, and assume the conditions in Theorem \ref{thm:state_bound} are satisfied. Consider two non-stationary policies $\{M_{t-L, i}, \ldots, M_{t-k, i}, \ldots, M_{t, i}\}$, and\\ $\{M_{t-L, i}, \ldots, \tilde{M}_{t-k, i}, \ldots, M_{t, i}\}$ which differ in exactly one control, then we have
    \begin{align*}
    &|g_{t, i}(M_{t-L, i}, \ldots, M_{t-k, i}, \ldots, M_{t, i}) - g_{t, i}(M_{t-L, i}, \ldots, \tilde{M}_{t-k, i}, \ldots, M_{t, i})| \\
    &\le \kappa^2(1-\gamma)^kD^2\sum_{l=1}^L\left(\|M_{t-k, i}^l, \tilde{M}_{t-k, i}^l\|\right).
    \end{align*}
\end{lemma}
\begin{proof}
Let $y_{t+1, i}$ denote $y_{t+1, i}(M_{t-L}, \ldots, M_{t-k, i}, \ldots, M_{t, i})$ and $\tilde{y}_{t+1, i}$ denote $y_{t+1, i}(M_{t-L, i},$ $\ldots, \tilde{M}_{t-k, i},\ldots,M_{t, i})$. By definition, 
\begin{align*}
\|y_{t+1, i} - \tilde{y}_{t+1, i}\|&=\|A_{t-k+1:t, i}B\sum_{j=0}^{2L-1}\left(M_{t-k, i}^{j-k} - \tilde{M}_{t-k, i}^{j-k}\right)\mathbbm{1}_{j-k\in [1, L]}w_{t-j, i}\|\\
&\le \kappa^2(1-\gamma)^kD\sum_{l=1}^L\|M_{t-k, i}^l - \tilde{M}_{t-k, i}^l\|,
\end{align*}
because the disturbances are bounded by Lemma \ref{lem:reset_bound}.
Similar to Lemma \ref{lem:truncation_bound}, we have
\begin{align*}
|c_{t, i}(y_{t+1, i}) -c_{t, i}(\tilde{y}_{t+1, i})| &\le \|\nabla c_{t, i}(\xi_{t, i})\|\|y_{t+1, i} - \tilde{y}_{t+1, i}\| \\&\le(2\beta D+b)\kappa^2(1-\gamma)^kD\sum_{l=1}^L\|M_{t-k, i}^l - \tilde{M}_{t-k, i}^l\|\\
&\le \kappa^2(1-\gamma)^kD^2\sum_{l=1}^L\|M_{t-k, i}^l - \tilde{M}_{t-k, i}^l\|.
\end{align*}
\end{proof}

\begin{lemma}\label{lem:f_Gradient}Suppose $M=\{M^l\}_{l=1}^L$ satisfies $\|M^l\|\le a(1-\gamma)^l$. Under conditions of Lemma \ref{lem:f_Lipschitz}, 
$$
\|\nabla_M g_{t, i}(M, \ldots, M)\|_F \le \frac{D^2\kappa^2dL}{\gamma}.$$
\end{lemma}
\begin{proof}
 Similar to the proof of Lemma 5.7 in \cite{agarwal2019online}, we derive absolute value bound on $\nabla_{M^l_{p, q}}g_{t, i}(M, \ldots, M)$ for all $l,p,q$.
\begin{align*}
|\nabla_{M^l_{p, q}}g_{t, i}(M, \ldots, M)| &\le \|\nabla c_{t, i}(y_{t, i}(M, \ldots, M))\|\left \|\frac{\partial y_{t, i}(M, \ldots, M)}{\partial M^l_{p, q}}\right \| \\
&\le D \left \|\frac{\partial y_{t, i}(M, \ldots, M)}{\partial M^l_{p, q}}\right \|,
\end{align*}
and the last inequality holds because $\|\nabla c_{t, i}(y_{t, i}(M, \ldots, M))\| \le 2\beta D+b \le D$. Moreover, 
\begin{align*}
\left \|\frac{\partial y_{t, i}(M, \ldots, M)}{\partial M^l_{p, q}}\right \| &\le \sum_{j=l}^{l+L}\left\|\frac{\partial A_{t-j+l+1:t-1}BM^l}{\partial M^l_{p, q}}\right\|\|w_{t-j}\|\le \frac{D\kappa^2}{\gamma}.
\end{align*}

% \begin{align*}
% y_{t, i}(M_{t-L, i}, \ldots, M_{t-1, i}) &= \sum_{k=1}^{2L}\left(\sum_{j=1}^{L}A_{t-j+1:t-1}BM^{k-j}_{t-j, i}\mathbbm{1}_{k-j\in[1, L]} + A_{t-k+1:t-1} \mathbbm{1}_{k\le L}\right)w_{t-k, i}
% \end{align*}
The lemma follows by summing over $l, p, q$.
\end{proof}

\subsubsection{Bounding the states and disturbances.}\label{sec:state_bound}
We give a universal upper bound on the state and disturbance magnitude in the following theorem. The central challenge in proving Theorem \ref{thm:state_bound} is the interdependence of the state and the disturbance, where the state can grow from a large disturbance, and the disturbance can scale with the state. We overcome this challenge by induction on both the episodes and the time steps.

In Lemma \ref{lem:state_recursion}, we analyze the effect of the large reset disturbance. If the state and the reset disturbance are bounded at the beginning of an episode, which is the case for the first episode, then we show by induction that the state magnitude undergoes two phases: the magnitude first rises from the effect of the reset disturbance, and then decays to a constant before the end of the episode. Importantly, this constant is the same as the state bound at the episode start. In Lemma \ref{lem:state_bound_first_epoch}, we derive the initial state bound, and in Lemma \ref{lem:reset_bound}, we compute an upper bound on the reset disturbance, completing the induction over episodes. We begin with the theorem statement below.
\begin{theorem} \label{thm:state_bound}
Suppose the sequence of policies $M_{1, 1}, \ldots, M_{T, N}$ satisfy $\|M_{t, i}^j\|\le a(1-\gamma)^j$ for some constant $a \ge 1$, where $M_{t, i}^j$ is the $j$-th matrix in $M_{t, i}$. By definition, $\|z_{1, i}\| \le 2R$ for all $i$, and define 
$$
D_1 = 2\kappa^2 R  + \frac{32\kappa^2 Lab}{\gamma},\ \ D_2 = 10D_1 + b.$$
Under Assumptions \ref{assumption:gradient}, \ref{assumption:reset}, \ref{assumption:sequential_stable}, if further we have $$\beta\le \frac{\gamma(1-\gamma)^2}{10^3\kappa^2 aL^2},\ \ L\ge \frac{\log(20^2 \kappa^2 La)}{\log(1/(1-\gamma))},\ \ T\ge 10L,\ \ \eta\le 1,$$ then $\|z_{t, 1}\| \le D_1$ for all $t\in [T]$, and for all $i\in[2, N]$, 
$$
\|z_{t, i}\| \le 8\kappa^2La(1-\gamma)^{t-1} D_2 + \frac{D_1}{4}\ \text{for } t\le L,\ \text{and }\|z_{t, i}\| \le D_1\ \text{for } t > L.$$
In particular, all states are bounded above by
$
\|z_{t, i}\|\le 8\kappa^2LaD_2 + D_1.$

\end{theorem}
\begin{proof}
The proof follows by combining Lemmas \ref{lem:state_recursion}, \ref{lem:state_bound_first_epoch}, \ref{lem:reset_bound}.
\end{proof}

\begin{lemma}\label{lem:state_recursion}
Suppose $M_{t, i}$ satisfies the condition in Theorem \ref{thm:state_bound} for all $t, i$. Assume that for some $D_1$, the last $2L$ states in the previous epoch has bounded magnitude:$\|z_{t, i-1}\| \le D_1$ for $t\in [T-2L, T]$, and the reset disturbance magnitude in the last epoch, $\|w_{T, i-1}\|$, has upper bound $D_2$ with $D_2 \ge \max\{D_1, 2R\}$. Then, if $\beta \le \frac{\gamma(1-\gamma)^2}{10^3\kappa^2 aL^2}$, $L \ge\frac{\log (32D_2\kappa^2 La/D_1)}{\log(1/(1-\gamma))}$ and $D_1 \ge 32\kappa^2Lab/\gamma$, we have

$$
\|z_{t, i}\| \le 8\kappa^2La(1-\gamma)^{t-1}D_2 + \frac{D_1}{4},\ \ t\le L \text{, and }\|z_{t, i}\| \le D_1,\ \ \ t\ge L.$$ 
\end{lemma}
\begin{proof}
We will prove by induction that for the first $4L$ time steps,
$$
\|z_{t, i}\| \le 8\kappa^2La(1-\gamma)^{t-1}D_2 + \frac{D_1}{4}.$$
The base case of $t=1$ holds by definition. Assume that the inductive hypothesis holds for all $s\le t-1$, for some $t\le 4L$.
For the new iteration $t$, the state can be written as 
\begin{align*}
z_{t, i} &= A_{1:t-1, i}z_{1, i} + \sum_{k=0}^{t-2+L}\left(\sum_{j=0}^{t-2}A_{t-j:t-1, i}BM^{k-j}_{t-1-j, i}\mathbbm{1}_{k-j\in[1, L]} + A_{t-k:t-1, i}\mathbbm{1}_{k\le t-2}\right)w_{t-1-k, i}.
\end{align*}   
The magnitude an be upper bounded as 
\begin{align*}
\|z_{t, i}\| &\le 2\kappa^2(1-\gamma)^{t-1}R + 2\kappa^2La\sum_{k=0}^{t-2+L}(1-\gamma)^k \|w_{t-1-k, i}\|\\
&\le 2\kappa^2(1-\gamma)^{t-1}R + 2\kappa^2La\sum_{k=0}^{t-2}(1-\gamma)^k \|w_{t-1-k, i}\| +2\kappa^2La (1-\gamma)^{t-1} D_2\\
&+2\kappa^2La\sum_{k=t}^{t-2+L}(1-\gamma)^k \|w_{T+t-1-k, i-1}\| \tag{$\|w_{T, i-1}\| \le D_2$}.
\end{align*}
Observe that the last term can be bounded as
\begin{align*}
2\kappa^2La \sum_{k=t}^{t-2+L}(1-\gamma)^k \|w_{T+t-1-k, i-1}\| &\le 2\kappa^2La \sum_{k=t}^{t-2+L}(1-\gamma)^k (2\beta\|z_{T+t-2-k, i-1}\|+b)\\
&\le 2\kappa^2La \sum_{k=t}^{t-2+L}(1-\gamma)^k (2\beta D_1+b)\\
&\le \frac{2\kappa^2La(1-\gamma)^t}{\gamma}(2\beta D_1+b)\\
&\le \frac{D_1}{16} + \frac{2\kappa^2Lab}{\gamma}.
\end{align*}
The first sum can be bounded using the inductive hypothesis:
\begin{align*}
2\kappa^2La\sum_{k=0}^{t-2}(1-\gamma)^k \|w_{t-1-k, i+1}\| 
&\le 2\kappa^2La\sum_{k=0}^{t-2}(1-\gamma)^k (2\beta\|z_{t-2-k, i+1}\|+b)\\
&\le 4\beta\kappa^2La\sum_{k=0}^{t-2}(1-\gamma)^k (8\kappa^2La(1-\gamma)^{t-3-k}D_2 + \frac{D_1}{4}) + \frac{2\kappa^2 Lab}{\gamma}\\
&=32\beta\kappa^4L^2a^2\sum_{k=0}^{t-2}(1-\gamma)^{t-3}D_2 + \frac{4\beta\kappa^2La}{\gamma}\frac{D_1}{4} + \frac{2\kappa^2 Lab}{\gamma}\\
&\le 128\beta\kappa^4L^3a^2(1-\gamma)^{t-3}D_2 + \frac{\beta\kappa^2La D_1}{\gamma} + \frac{D_1}{16} \tag{$t\le 4L$}\\
&\le \frac{1}{2}\kappa^2La(1-\gamma)^{t-1}D_2 + \frac{D_1}{8},
\end{align*}
where the last step holds because $\beta \le \frac{(1-\gamma)^2}{256\kappa^2L^2a}$ and $\beta \kappa^2La/\gamma\le 1/16$. We also note that at start of the epoch, we have $\|\nabla f_{1, i}(x_{1, i})\| \le \beta\|x_{1, i}\|+b\le \beta\|z_{1, i}\|+b$, so in the expansion above, take $z_{0, i} = z_{1, i}$. Adding everything together, 
\begin{align*}
\|z_{t, i}\| &\le 2\kappa^2(1-\gamma)^{t-1}R + \frac{1}{2}\kappa^2La(1-\gamma)^{t-1}D_2 + \frac{D_1}{8} +2\kappa^2La (1-\gamma)^{t-1} D_2 +\frac{D_1}{16} + \frac{2\kappa^2Lab}{\gamma}\\
&\le 8\kappa^2La(1-\gamma)^{t-1}D_2 + \frac{D_1}{8} +\frac{D_1}{8}\\
&\le 8\kappa^2La(1-\gamma)^{t-1}D_2 + \frac{D_1}{4}.
\end{align*}
We conclude that the inductive hypothesis holds for all $t\le 4L$. We proceed to bound the states for $t \ge 4L$ by again using induction. Since we set 
$$
L \ge\frac{\log (32D_2\kappa^2 La/D_1)}{\log(1/(1-\gamma))},\ \ \ 8\kappa^2La(1-\gamma)^LD_2\le \frac{D_1}{4},$$
hence $\|z_{t, i}\| \le D_1$ for $t\in [L+1, 4L]$. Assume that $\|z_{t, i}\|\le D_1$ for all $s\in [L+1, t]$, for some $t \ge 4L$.
Consider the time step $z_{t+1, i}$, which can be decomposed as 
\begin{align*}
z_{t+1, i} &= A_{t+1-L:t, i}z_{t+1-L, i} + \sum_{k=0}^{2L-1}\left(\sum_{j=0}^{L-1}A_{t-j+1:t, i}BM^{k-j}_{t-j, i}\mathbbm{1}_{k-j\in[1, L]} + A_{t-k+1:t, i}\mathbbm{1}_{k\le L-1}\right)w_{t-k, i}. 
\end{align*}
The state then admits the following upper bound
\begin{align*}
\|z_{t+1, i}\| &\le \kappa^2(1-\gamma)^L \|z_{t+1-L, i}\| + 2aL\kappa^2\sum_{k=0}^{2L-1}(1-\gamma)^k\|w_{t-k, i}\|\\
&\le \kappa^2(1-\gamma)^L \|z_{t+1-L, i}\| + 2aL\kappa^2\sum_{k=0}^{2L-1}(1-\gamma)^k(2\beta\|z_{t-1-k, i}\| + b)\\
&\le \kappa^2(1-\gamma)^L D_1 + 2aL\kappa^2\sum_{k=0}^{2L-1}(1-\gamma)^k(2\beta D_1 + b) \tag{inductive hypothesis}\\
&\le \kappa^2(1-\gamma)^L D_1 + \frac{4\beta aL\kappa^2 D_1}{\gamma} + \frac{2 aL\kappa^2 b}{\gamma}\\
&\le  \frac{D_1}{4} + \frac{D_1}{16} + \frac{D_1 }{16} \le D_1.
\end{align*}
Since this induction step can be applied to any $t\ge 4L$, we conclude that $\|z_{t, i}\|\le D_1$ for all $t\ge 4L$.
\end{proof}

We can use the above lemma to bound the state magnitude by computing the quantities $D_1, D_2$.

\begin{lemma}\label{lem:state_bound_first_epoch}
Suppose $\beta$, $L$, $M_{t, i}$ satisfy the conditions in Lemma \ref{lem:state_recursion}, then for all $t\in[T]$, the states are bounded in the first epoch:
$$\|z_{t, 1}\| \le 2\kappa^2\|z_{1, 1}\| + \frac{32\kappa^2Lab}{\gamma} \le 2\kappa^2R + \frac{32\kappa^2Lab}{\gamma}:= D_1,
$$
and the controls have upper bound 
$
\|u_{t, 1}\| \le \|z_{1, 1}\| + \frac{2ab}{\gamma}.$
\end{lemma}
\begin{proof}
We first show that the states in the first epoch satisfy $\|z_t\| \le 2\kappa^2\|z_1\| + \frac{8\kappa^2Lab}{\gamma}$
by induction.
The state at time $t$ can be expressed as
\begin{align*}
z_{t, 1} &= A_{1:t-1, 1}z_{1, 1} + \sum_{k=0}^{t-2+L}\left(\sum_{j=0}^{t-2}A_{t-j:t-1, 1}BM^{k-j}_{t-1-j, 1}\mathbbm{1}_{k-j\in[1, L]} + A_{t-k:t-1, 1}\mathbbm{1}_{k\le t-2}\right)w_{t-1-k, 1},
\end{align*}
where $\|w_{t, 1}\| = 0$ for $t \le 0$.
We can bound the state magnitude as 
\begin{align*}
\|z_{t, 1}\| &\le \kappa^2(1-\gamma)^{t-1}\|z_{1, 1}\| + \sum_{k=0}^{t-2+L}(aL+1)\kappa^2(1-\gamma)^k \|w_{t-1-k, 1}\|\\
&\le \kappa^2(1-\gamma)^{t-1}\|z_{1, 1}\| + 2aL\kappa^2\sum_{k=0}^{t-2+L}(1-\gamma)^k (2\beta\|z_{t-2-k, 1}\| + b)
\end{align*}
Now, the base case of our induction clearly holds: $\|z_{1, 1}\| \le 2\kappa^2\|z_{1, 1}\| + \frac{8\kappa^2Lab}{\gamma}.$
Assume that for some $t$, the statement holds for all $z_{s, 1}$, where $s\in[1, t-1]$. Then the following holds for $z_{t, 1}$:
\begin{align*}
\|z_{t, 1}\| &\le \kappa^2(1-\gamma)^{t-1}\|z_{1, 1}\| + 2aL\kappa^2\sum_{k=0}^{t-2+L}(1-\gamma)^k (2\beta(2\kappa^2\|z_{1, 1}\| + \frac{8\kappa^2Lab}{\gamma}) + b)\\
&\le \kappa^2(1-\gamma)^{t-1}\|z_{1, 1}\| + \frac{8aL\kappa^4\beta}{\gamma}\|z_{1, 1}\| +\frac{32aL\kappa^2\beta}{\gamma} \frac{\kappa^2Lab}{\gamma}+ \frac{2aL\kappa^2 b}{\gamma}.
\end{align*}
By the condition on $\beta$, we have
$\frac{8aL\kappa^2\beta}{\gamma} \le \frac{1}{40},
$
and therefore 
\begin{align*}
\|z_{t, 1}\| &\le \kappa^2(1-\gamma)^{t-1}\|z_{1, 1}\| + \frac{\kappa^2}{40}\|z_{1, 1}\| +\frac{1}{10}\frac{\kappa^2Lab}{\gamma}+ \frac{4aL\kappa^2 b}{\gamma}\le 2\kappa^2\|z_{1, 1}\|+ \frac{8aL\kappa^2 b}{\gamma}.
\end{align*}
We conclude that the induction hypothesis holds for all $t$.
% Now, for the rest of the states, observe that the system evolution over the last $L$ time steps is:
% \begin{align*}
% z_{t+1, 1} &= A_{t+1-L:t, 1}z_{t+1-L, 1} + \sum_{k=0}^{2L-1}\left(\sum_{j=0}^{L-1}A_{t-j+1:t, 1}BM^{k-j}_{t-j, 1}\mathbbm{1}_{k-j\in[1, L]} + A_{t-k+1:t, 1}\mathbbm{1}_{k\le L-1}\right)w_{t-k, 1}. 
% \end{align*}
% We use an induction based on the recurrence relation above to bound the rest of the states. For the base case $$ Define $c = \frac{8\kappa^2Lab}{\gamma}$, then
% \begin{align*}
% \|z_{t+1}\| &\le \kappa^2(1-\gamma)^{L+1}\|z_{t-L}\| + \sum_{k=0}^{2L} (\kappa^2(1-\gamma)^k + \kappa^2(1-\gamma)^kLa) (2\beta\|z_{t-k}\| + b)\\
% &\le \kappa^2(1-\gamma)^{L+1}\|z_{t-L}\| + 2\kappa^2La\sum_{k=0}^{2L} (1-\gamma)^k (2\beta(2\kappa^2\|z_1\| + c) + b)\\
% &\le \kappa^2(1-\gamma)^{L+1}(2\kappa^2\|z_1\| + c) + \frac{8\kappa^4La\beta}{\gamma}\|z_{1, 1}\| + \frac{4\kappa^2La\beta}{\gamma}c + \frac{c}{4}\\
% &\le 2\kappa^4(1-\gamma)^{L+1}\|z_1\| + \kappa^2(1-\gamma)^{L+1}c+ \frac{8\kappa^4La\beta}{\gamma}\|z_{1, 1}\| + \frac{4\kappa^2La\beta}{\gamma}c + \frac{c}{4}.
% \end{align*}
% By our choice of $L$, $\kappa^2(1-\gamma)^{L+1} \le 1/2$, and by the upper bound on $\beta$, $\frac{8\kappa^2La\beta}{\gamma} \le 1/100$. Therefore,
% \begin{align*}
% \|z_{t+1, 1}\| \le \kappa^2\|z_{1, 1}\| + \frac{c}{2} + \frac{\kappa^2}{100}\|z_{1, 1}\| + \frac{c}{100} + \frac{c}{4} \le 2\kappa^2\|z_{1, 1}\| + c.
% \end{align*}
% Since this inductive step can be applied for all $t\ge 4L$, the inductive hypothesis is true for all $z_{t, 1}$. 

Moreover, we can bound the magnitude of the controls by
\begin{align*}
\|u_{t, 1}\| &\le \sum_{l=1}^L \|M_{t, 1}^l\|\|w_{t-l, 1}\|\\ &\le a\sum_{l=1}^L (1-\gamma)^l (2\beta\|z_{t-l-1, 1}\|+b)\\
&\le a\sum_{l=1}^L (1-\gamma)^l (2\beta(2\kappa^2\|z_{1, 1}\|+\frac{8aL\kappa^2 b}{\gamma})+b)\\
&\le \frac{4a\beta\kappa^2}{\gamma}\|z_{1, 1}\| + \frac{16\beta a^2}{\gamma} \frac{L\kappa^2 b}{\gamma}+ \frac{ab}{\gamma}\\
&\le \|z_{1, 1}\| + \frac{2ab}{\gamma}.
\end{align*}
\end{proof}

\begin{lemma}\label{lem:reset_bound}
Suppose $\eta \le 1$, $T\ge 10L$, and the conditions of Lemma \ref{lem:state_bound_first_epoch} are satisfied. Define $D_1$ as in Lemma \ref{lem:state_bound_first_epoch},
then the reset disturbance can be bounded as
$
\|w_{T, i}\| \le 10D_1 + b
$ for all $i$.
\end{lemma}
\begin{proof}
By definition, the reset disturbance satisfies
\begin{align*}
\|w_{T, i}\| &\le \|x_{1, i+1} - ((1-\delta)x_{T, i} - \eta \nabla f_{T-1, i}(x_{T-1, i}) + \bar{u}_{T, i})\| + \|x_{1, i+1} - x_{T, i}\| + \|H_{T, i}(x_{T, i} - x_{T-1, i})\|\\
&\le 2R + \|x_{T, i}\| + \eta\|\nabla f_{T-1, i}(x_{T-1, i})\| + \|u_{T, i}\| + \|x_{T-1, i}\| + \|x_{T, i}\| + \|x_{T-1, i}\|. \tag{$\|H_{T, i}\|\le 1$}
\end{align*}
We will show the lemma by induction on $i$. For the base case of $i = 1$, we have
\begin{align*}
\|w_{T, 1}\| &\le 3R + 4D_1 + \frac{2ab}{\gamma}+\eta\|\nabla f_{T-1,1}(x_{T-1, 1})\| \tag{Lemma \ref{lem:state_bound_first_epoch}}\\
&\le 3R + 4D_1 + \frac{2ab}{\gamma}+ 2\beta\|z_{T-2, 1}\|+ b \tag{$\eta \le 1$}\\
&\le 3R + 4D_1 + \frac{2ab}{\gamma}+ 2\beta D_1+b\\
&\le 10D_1 + b.
\end{align*}
Suppose the inductive hypothesis holds for $i\le n$, then under the conditions of Lemma \ref{lem:state_recursion}, with $D_2 = 10D_1 + b$, the state eventually goes back to be bounded by $D_1$ at time $T-1$ in epoch $n+1$. Therefore,
\begin{align*}
\|w_{T, n+1}\| &\le 2R + 4D_1 + \eta\|\nabla f_{T-1, n+1}(x_{T-1, n+1})\| + \|u_{T, n+1}\|.
\end{align*}
Note that the control magnitude satisfies
\begin{align*}
\|u_{T, n+1}\| &\le a\sum_{l=1}^L (1-\gamma)^l\|w_{T-l, n+1}\|\\
&\le a\sum_{l=1}^L (1-\gamma)^l(2\beta\|z_{T-l-1, n+1}\| +b)\\
&\le \frac{2a\beta D_1}{\gamma} + \frac{ab}{\gamma} \le D_1.
\end{align*}
So we have 
\begin{align*}
\|w_{T, n+1}\| &\le 2R + 5D_1 + 2\beta\|z_{T-1, n+1}\|+b\\
&\le 2R + 5D_1 + 2\beta D_1+b \\
&\le 10D_1 + b.
\end{align*}
We conclude that the claim holds for all $i\in [N]$.
\end{proof}

\subsection{Proof of Theorem \ref{thm:smooth}.}\label{sec:smooth_proof}
The proof of Theorem \ref{thm:smooth} is similar to that of Theorem \ref{thm:main} and uses the same regret decomposition. Given a universal state magnitude upper bound, we can again bound the approximation error terms. For analyzing the expected regret of OCOwM, we can invoke the guarantee for BCO with memory, Theorem \ref{thm:bco_main} in the Appendix.
\begin{proof}
Define the idealized state and idealized cost as in the proof of Theorem \ref{thm:main}. We can again break down the expected meta-regret into three parts,
\begin{align}
&\E\left[\sum_{i=1}^N \sum_{t=1}^T c_{t, i}(z_{t, i})\right] - \min_{M\in \M}\sum_{i=1}^N\sum_{t=1}^T c_{t, i}(z_{t, i}(M)) \nonumber\\
&\le \E\left[\sum_{i=1}^N \sum_{t=1}^T c_{t, i}(z_{t, i}) - \sum_{i=1}^N \sum_{t=1}^T g_{t, i}(\widetilde{M}_{t-L, i}, \ldots, \widetilde{M}_{t-1, i})\right] \label{eq:6}\\&+ \E\left[\sum_{i=1}^N \sum_{t=1}^T g_{t, i}(\widetilde{M}_{t-L, i}, \ldots, \widetilde{M}_{t-1, i})\right] - \min_{M\in \M}\sum_{i=1}^N \sum_{t=1}^T g_{t, i}(M, \ldots, M)\label{eq:8}\\
&+ \min_{M\in \M}\sum_{i=1}^N \sum_{t=1}^T g_{t, i}(M, \ldots, M) - \min_{M\in \M}\sum_{i=1}^N \sum_{t=1}^T c_{t, i}(z_{t, i}(M)) \label{eq:7}.
\end{align}
Similar to the proof of Theorem \ref{thm:main}, assume $\beta \le \beta_0$, and define $D$ to be the same as in the proof. Then the conditions of Theorem \ref{thm:state_bound} is satisfied, and both the states and idealized states are bounded by $D$. In addition, we have $\beta D+b\le D$.
By the same reasoning as in Lemma \ref{lem:truncation_bound}, we have 
\begin{align*}
    \sum_{i=1}^N \sum_{t=1}^T c_{t, i}(z_{t, i}) - \sum_{i=1}^N \sum_{t=1}^T g_{t, i}(\widetilde{M}_{t-L, i}, \ldots, \widetilde{M}_{t-1, i}) &\le (\beta D+b)\left(TN\kappa^2(1-\gamma)^LD + \frac{\kappa^2D}{\gamma}\right) \\
    &\le \kappa^2D^2\left(TN(1-\gamma)^L + \frac{1}{\gamma}\right).
\end{align*}
Let $M^*  = \argmin_{M\in\M}\sum_{i=1}^N \sum_{t=1}^T c_{t, i}(z_{t, i}(M))$ denote the best DAC controller in hindsight,
\begin{align*}
    (\ref{eq:7})\le \sum_{i=1}^N \sum_{t=1}^T g_{t, i}(M^*, \ldots, M^*) - \sum_{i=1}^N \sum_{t=1}^T c_{t, i}(z_{t, i}(M^*)) \le  \kappa^2D^2\left(TN(1-\gamma)^L + \frac{1}{\gamma}\right).
\end{align*}
Observe that since $c_{t, i}$ is Lipschitz and smooth by assumptions on $f_{t, i}$, $g_{t, i}$ is also Lipschitz and smooth. 
Setting $\{\eta^g_{t, i}\}, \delta_g$ properly, we can use Theorem \ref{thm:bco_main} to bound (\ref{eq:8}) as:
$$
(\ref{eq:8}) \le \tilde{O}((NT)^{3/4}),$$ where the logarithmic factors of $N$ and $T$ are due to the dependence on $L = \Theta(\log NT)$.
\end{proof}

%% file: conclusion.tex
\section{Conclusion.}

This manuscript proposes a framework for optimization whose goal is to learn the best optimization algorithm from experience, and gives an algorithmic methodology using feedback control for this meta-optimization problem. In contrast to well-studied connections between optimization and Lyapunov stability, our approach builds upon new techniques from the regret-based online nonstochastic control framework. We derive new efficient algorithms for meta-optimization using recently proposed control methods, and prove regret bounds for them. 

% We present several applications, including efficient method for meta-optimization of the momentum and learning rate parameters. 

An interesting direction for future investigation is extending the work on nonstochastic control of time varying and nonlinear dynamical systems to our setting of meta-optimization. In particular, it is interesting to explore the implications of low adaptive and dynamic regret in meta-optimization. Another area of investigation is the connection between meta-optimization using control, and adaptive methods: do the solutions proposed by the control approach have connections to adaptive regularization-based algorithms? Many intriguing problems arise in this new intersection of optimization and nonstochastic control.

%% file: bandit_ocowm.tex
\section{Bandit Perturbation Control}\label{sec:bco}
To enhance our results from meta-optimization of quadratic functions to more general smooth objectives, we encounter a difficulty: the dynamical system formulation of meta optimization now has changing and potentially unknown dynamics. 
To cope with this difficulty, we resort to techniques in online nonstochastic control that were developed to deal with unknown and changing dynamical systems. However, a crucial difference makes our derivation easier and different: in control, the system must be identified to recover the perturbations, or approximations thereof. In meta-optimization, the perturbations are known to the controller! 

The foundation of our algorithm for convex smooth meta-optimization is bandit nonstochastic control. Due to space constraints, in this appendix we give an overview of results in bandit convex optimization with memory, which is the core technique for bandit nonstochastic control. 

% foundation of our algorithm for convex smooth meta-optimization.

% \subsection{Bandit convex optimization with memory}

The setting of online convex optimization with memory \citep{anava2015online} is identical to that of standard online convex optimization, except for the following crucial difference: the cost functions depend on a history of points, i.e. 
$ f_t(x_t,...,x_{t-H}) $.
In the bandit setting of OCO with memory, called Bandit Convex Optimization (BCO) with memory, the only information available to the learner after each iteration is the loss value itself, a scalar. 

The framework of BCO with memory is used to capture time dependence of the reactive environment. The adversary picks loss functions $f_t$ with bounded memory $H$ of the decision makers' previous predictions.  The goal is to minimize regret, defined as:
\begin{align*}
\mathop{\text{Regret}} = \mathop{\mathbb{E}} \left[\sum_{t=1}^T f_t(x_{t-{H}:t})\right] - \min_{x^{\star} \in \mathcal{K}} \sum_{t=1}^T f_t(x^{\star}, \ldots, x^{\star}),
\end{align*}
where we denote $x_{t-{H}:t} = (x_{t-H}, \ldots, x_t)$ and $x_1, \ldots, x_T$ are the predictions of algorithm $\mathcal{A}$.

\begin{algorithm}[t]
\caption{BCO with Memory}
\label{alg:bco-w-mem}
\begin{algorithmic}[1]
    \State \textbf{Input:} $\mathcal{K}$, $T$, $H$, $\{ \eta_t \}$ and $\delta$.
    \State Initialize $x_{1} = \cdots = x_{H} \in \mathcal{K}_{\delta}$ arbitrarily, sample $u_{1}, \ldots, u_{H} \in_{\mathbf{R}} \mathbb{S}_{1}^{d}$.
    \State Set $y_i = x_i + \delta u_i, g_i = 0$ for $i = 1, \ldots, H$.
%    \STATE Set $g_i = 0$ for $i = 1, \ldots, H$
 %   \STATE Predict $y_i$ for $i = 1, \ldots, H$
    \For {$t = H, \ldots, T$}
        \State Predict $y_t$, and suffer loss $f_t(y_{t-{H}:t})$.
        \State Store $g_t = \frac{d}{\delta}f_t(y_{t-{H}:t}) \sum\limits_{i=1}^{{H}} u_{t-i}$, and set $x_{t+1} = \mathop{\Pi}\limits_{\mathcal{K}_{\delta}} \left[x_{t} - \eta_{t} \; g_{t-H} \right]$.
        \State Sample $u_{t+1} \in_{\text{R}} \mathbb{S}_{1}^d$ , 
 set $y_{t+1} = x_{t+1} + \delta u_{t+1}$.
    \EndFor
%    \STATE return
\end{algorithmic}
\end{algorithm}

Algorithm \ref{alg:bco-w-mem} for BCO with memory is developed in \cite{gradu2020non}. The main performance guarantee for this algorithm is given in Theorem \ref{thm:bco_main}, proven in \cite{gradu2020non}.
For the settings of Theorem \ref{thm:bco_main}, we assume that the loss functions $f_t$ are convex with respect to $x_{t-{H}:t}$, $G$-Lipschitz, $\beta$-smooth, and bounded. We can assume without loss of generality that the loss functions are bounded by $1$ to simplify our calculations; the regret scales linearly with the function value upper bound.

\begin{theorem}\label{thm:bco_main}
Setting step sizes $\eta_t = \Theta(t^{-3/4} H^{-3/2} d^{-1} D^{2/3} G^{-2/3} \beta^{-1/2})$ and perturbation constant $\delta = \Theta(T^{-1/4} H^{-1/2} D^{1/3} G^{-1/3})$, Algorithm \ref{alg:bco-w-mem} produces a sequence $\{y_t\}_{t=0}^{T}$ that satisfies:
\begin{align*}
\mathop{\text{Regret}} \leq \mathcal{O}\left(T^{3/4} H^{3/2} d D^{4/3} G^{2/3} \beta^{1/2}\right)
\end{align*} 

In particular, $\text{\emph{Regret}} \leq \mathcal{O}\left(T^{3/4}\right)$.
\end{theorem}

We can apply the BCO with Memory algorithm in the nonstochastic control framework, to control an unknown LTV system in the bandit setting with known perturbations. The approach is to design a disturbance action controller and train it using the algorithm for BCO with memory. The full algorithm and guarantees are derived in \cite{gradu2020non}.